\PassOptionsToPackage{unicode}{hyperref}
\PassOptionsToPackage{hyphens}{url}
\PassOptionsToPackage{dvipsnames,svgnames,x11names}{xcolor}
\documentclass[
  12pt]{article}

\usepackage{amsmath,amssymb}
\usepackage{iftex}

\newtheorem{assumption}{Assumption}
 
\newtheorem{theorem}{Theorem}
\newtheorem{lemma}[theorem]{Lemma}

\newcommand{\BlackBox}{\rule{1.5ex}{1.5ex}}  % end of proof
\ifdefined\proof
    \renewenvironment{proof}{\par\noindent{\bf Proof\ }}{\hfill\BlackBox\\[2mm]}
\else
    \newenvironment{proof}{\par\noindent{\bf Proof\ }}{\hfill\BlackBox\\[2mm]}
\fi

\ifPDFTeX
  \usepackage[T1]{fontenc}
  \usepackage[utf8]{inputenc}
  \usepackage{textcomp} % provide euro and other symbols
\else % if luatex or xetex
  \usepackage{unicode-math}
  \defaultfontfeatures{Scale=MatchLowercase}
  \defaultfontfeatures[\rmfamily]{Ligatures=TeX,Scale=1}
\fi
\usepackage{lmodern}
\ifPDFTeX\else  
\fi
\IfFileExists{upquote.sty}{\usepackage{upquote}}{}
\IfFileExists{microtype.sty}{% use microtype if available
  \usepackage[]{microtype}
  \UseMicrotypeSet[protrusion]{basicmath} % disable protrusion for tt fonts
}{}
\makeatletter
\@ifundefined{KOMAClassName}{% if non-KOMA class
  \IfFileExists{parskip.sty}{%
    \usepackage{parskip}
  }{% else
    \setlength{\parindent}{0pt}
    \setlength{\parskip}{6pt plus 2pt minus 1pt}}
}{% if KOMA class
  \KOMAoptions{parskip=half}}
\makeatother
\usepackage{xcolor}
\makeatletter
\ifx\paragraph\undefined\else
  \let\oldparagraph\paragraph
  \renewcommand{\paragraph}{
    \@ifstar
      \xxxParagraphStar
      \xxxParagraphNoStar
  }
  \newcommand{\xxxParagraphStar}[1]{\oldparagraph*{#1}\mbox{}}
  \newcommand{\xxxParagraphNoStar}[1]{\oldparagraph{#1}\mbox{}}
\fi
\ifx\subparagraph\undefined\else
  \let\oldsubparagraph\subparagraph
  \renewcommand{\subparagraph}{
    \@ifstar
      \xxxSubParagraphStar
      \xxxSubParagraphNoStar
  }
  \newcommand{\xxxSubParagraphStar}[1]{\oldsubparagraph*{#1}\mbox{}}
  \newcommand{\xxxSubParagraphNoStar}[1]{\oldsubparagraph{#1}\mbox{}}
\fi
\makeatother

\usepackage{longtable,booktabs,array}
\usepackage{calc} % for calculating minipage widths
\usepackage{etoolbox}
\makeatletter
\patchcmd\longtable{\par}{\if@noskipsec\mbox{}\fi\par}{}{}
\makeatother
\IfFileExists{footnotehyper.sty}{\usepackage{footnotehyper}}{\usepackage{footnote}}
\makesavenoteenv{longtable}
\usepackage{graphicx}
\makeatletter
\def\maxwidth{\ifdim\Gin@nat@width>\linewidth\linewidth\else\Gin@nat@width\fi}
\def\maxheight{\ifdim\Gin@nat@height>\textheight\textheight\else\Gin@nat@height\fi}
\makeatother
\setkeys{Gin}{width=\maxwidth,height=\maxheight,keepaspectratio}
\makeatletter
\def\fps@figure{htbp}
\makeatother

\makeatletter
\@ifpackageloaded{caption}{}{\usepackage{caption}}
\AtBeginDocument{%
\ifdefined\contentsname
  \renewcommand*\contentsname{Table of contents}
\else
  \newcommand\contentsname{Table of contents}
\fi
\ifdefined\listfigurename
  \renewcommand*\listfigurename{List of Figures}
\else
  \newcommand\listfigurename{List of Figures}
\fi
\ifdefined\listtablename
  \renewcommand*\listtablename{List of Tables}
\else
  \newcommand\listtablename{List of Tables}
\fi
\ifdefined\figurename
  \renewcommand*\figurename{Figure}
\else
  \newcommand\figurename{Figure}
\fi
\ifdefined\tablename
  \renewcommand*\tablename{Table}
\else
  \newcommand\tablename{Table}
\fi
}
\@ifpackageloaded{float}{}{\usepackage{float}}
\floatstyle{ruled}
\@ifundefined{c@chapter}{\newfloat{codelisting}{h}{lop}}{\newfloat{codelisting}{h}{lop}[chapter]}
\floatname{codelisting}{Listing}

\makeatother
\makeatletter
\@ifpackageloaded{caption}{}{\usepackage{caption}}
\@ifpackageloaded{subcaption}{}{\usepackage{subcaption}}
\makeatother

\ifLuaTeX
  \usepackage{selnolig}  % disable illegal ligatures
\fi
\usepackage[]{natbib}
\usepackage{bookmark}

\IfFileExists{xurl.sty}{\usepackage{xurl}}{} % add URL line breaks if available
\hypersetup{
  pdftitle={Title},
  pdfauthor={Author 1; Author 2},
  pdfkeywords={3 to 6 keywords, that do not appear in the title},
  colorlinks=true,
  linkcolor={blue},
  filecolor={Maroon},
  citecolor={Blue},
  urlcolor={Blue},
  pdfcreator={LaTeX via pandoc}}

\newcommand{\anon}{1}

\begin{document}

\def\spacingset#1{\renewcommand{\baselinestretch}%
{#1}\small\normalsize} \spacingset{1}

%%%%%%%%%%%%%%%%%%%%%%%%%%%%%%%%%%%%%%%%%%%%%%%%%%%%%%%%%%%%%%%%%%%%%%%%%%%%%%

\if1\anon
{
  \title{\bf Robust Reinforcement Learning under Diffusion Models for Data with Jumps}
  \author{Chenyang Jiang\\
    Department of Statistics, University of Wisconsin-Madison\\
    Madison, WI 53706, USA \\
    cjiang@wisc.edu \\
    and\\
    Donggyu Kim\\
    Department of Economics, University of California, Riverside \\
    Riverside, CA 92521, USA\\
    donggyu.kim@ucr.edu\\
    and\\
    Alejandra Quintos\thanks{Supported by the Office of the Vice Chancellor for Research and Graduate Education at the University of Wisconsin-Madison with funding from the Wisconsin Alumni Research Foundation.}\\
    Department of Statistics, University of Wisconsin-Madison\\
    Madison, WI 53706, USA\\
    alejandra.quintos@wisc.edu\\
    and\\
    Yazhen Wang \\
    Department of Statistics, University of Wisconsin-Madison\\
    Madison, WI 53706, USA\\
    yzwang@stat.wisc.edu}
  \maketitle
} \fi
       
\if0\anon
{
  \bigskip
  \bigskip
  \bigskip
  \begin{center}
    {\LARGE\bf Robust Reinforcement Learning under Diffusion Models for Data with Jumps}
\end{center}
  \medskip
} \fi

\bigskip

\allowdisplaybreaks

\begin{abstract}
Reinforcement Learning (RL) has demonstrated considerable success in tackling complex decision-making problems across diverse domains. However, continuous-time RL remains challenging, particularly when the underlying state dynamics are modeled by stochastic differential equations (SDEs) with jump components. In this paper, we address these challenges by proposing the Mean-Square Bipower Variation Error (MSBVE) algorithm, which enhances robustness and convergence in environments with significant stochastic noise and jumps. We begin by revisiting the widely used Mean-Square Temporal Difference Error (MSTDE) algorithm and highlight its susceptibility to bias when jump components are present. 
In contrast, the proposed MSBVE algorithm minimizes the bipower variation error that has the form of a product of non-overlapping absolute differences, allowing it to more accurately estimate value functions in the presence of jumps. Theoretical analysis and numerical studies confirm that MSBVE outperforms MSTDE in jump-diffusion settings. These results emphasize the critical role of alternative error metrics in advancing the resilience and accuracy of RL algorithms in continuous-time stochastic systems with jumps.
\end{abstract}

\noindent%
\vfill

\newpage
\spacingset{1.8} % DON'T change the spacing!

\section{Introduction}\label{sec-intro}

Reinforcement Learning (RL) has emerged as a powerful framework for solving intricate decision-making challenges across a variety of domains, including robotics, finance, and healthcare. In RL, an agent interacts with a dynamic environment and learns to select actions that maximize cumulative rewards over time. A central objective in this paradigm is the accurate estimation of the value function, which quantifies the expected return associated with each state under a given policy. Precise value function estimation is crucial for policy improvement and long-term performance.

Several algorithms for estimating the value function within the RL framework have been developed for purely continuous diffusion processes. For instance, \cite{Doya:2000} introduced the temporal-difference (TD) algorithm tailored for the continuous-time deterministic case, while \cite{jiaandzhou2022} elucidated the convergence properties of the Mean Square TD Error (MSTDE) algorithm within continuous-time stochastic settings. These algorithms endeavor to estimate the value function by iteratively updating a parameter vector using distinct error metrics, thereby shedding light on their convergence behavior under diverse noise conditions. The optimal parameter derived from these algorithms ensures that the estimated value function closely approximates the true value function.

However, in practice, observed data often exhibit jumps, which makes the pure continuous diffusion component unobservable. These discontinuities pose significant challenges for standard RL algorithms. In many applications, especially in finance, the primary interest lies in learning models based on the latent continuous component. For instance, volatility processes are crucial for quantifying investment risk. Since jumps typically reflect unpredictable news events, they are often treated as noise, and attention is focused on estimating the volatility of the continuous part.
At the very least, given the fundamentally different characteristics of the continuous and jump components, they should be learned separately.
See \citet{andersen2007roughing, corsi2010threshold, oh2024robust}.
The latent continuous process captures the typical dynamics that we aim to learn.
Therefore, it is crucial to develop robust RL estimation methods that can accurately recover the latent continuous component in the presence of jumps.

 In this paper, we delve into the task of estimating the value function within RL frameworks operating in continuous-time settings, where state dynamics are delineated by stochastic differential equations (SDEs) potentially incorporating jump components.  
 We begin by reviewing the Mean Square Temporal Difference Error (MSTDE) algorithm, which aims to estimate the value function by minimizing the mean-square TD error. However, we highlight its limitations in environments with jump-driven dynamics.

 To overcome these challenges, we propose the Mean Square Bipower Variation Error (MSBVE) algorithm, which minimizes the bipower variation error defined as a product of non-overlapping absolute differences. This approach enhances robustness to stochastic noise and discontinuities, which leads to more accurate value function estimation in the presence of jumps.
 Through rigorous theoretical analysis and comprehensive simulation studies,  we demonstrate the effectiveness of the MSBVE algorithm in accurately estimating the value function under challenging conditions characterized by substantial variability arising from jump components. Our results clearly show that MSBVE outperforms the MSTDE algorithm, particularly in environments where substantial jump-induced noise hinders MSTDE's convergence. These findings highlight the critical importance of adopting alternative error metrics, such as the bipower variation error, to enhance the robustness and accuracy of RL algorithms in continuous-time settings with jump components.

 The key distinction of our work lies in its emphasis on robustness.
 While recent studies have primarily focused on deriving optimal control policies in jump-diffusion settings, our approach aims to develop an algorithm that ensures the estimated value function remains resilient to jump-induced noise. 
This leads to a more stable and reliable estimation framework for general SDEs with jumps and helps learn the latent continuous diffusion process. 
 In contrast, \cite{gao2024reinforcementlearningjumpdiffusionsfinancial} and \cite{Bo2024} employ randomized control methods alongside q-learning to derive the optimal policy distribution for general SDEs with jumps, differing mainly in their choice of entropy terms. \cite{Bender2023} narrows its focus to the mean-variance portfolio selection problem, providing theoretical optimal solutions but not addressing general SDEs with jumps. \cite{guo2022} introduces the Greedy Least-Square algorithm and derives sublinear regret bounds, though its model assumes that only the drift term depends on the state and control, excluding the diffusion and jump components. Lastly, \cite{Denkert2024} replaces control with a random point process, applying an actor-critic algorithm to solve a randomized problem that ultimately leads to the optimal value functions for the original system. Unlike these works, which concentrate on finding the optimal control under jump diffusions, our research addresses the challenge of value function estimation under jump dynamics and emphasizes robustness against jump noise.

 The remainder of this paper is structured as follows. Section \ref{sec-2} provides an overview of the setup for the RL problem within continuous-time settings. In Section \ref{sec-3}, we present the detailed descriptions of the MSTDE and MSBVE algorithms. Section \ref{simulation} reports the simulation results and analyzes the performance of these algorithms. In Section \ref{application}, we apply the MSBVE algorithm to the mean-variance portfolio selection problem and compare its performance with MSTDE using the real intraday trading data. Finally, we draw conclusions in Section \ref{sec-6}, summarizing the key findings and implications of our study. All the detailed proofs are given in appendix A.

\section{Model Setup} \label{sec-2}
In a general RL problem, the state space and action space are denoted by $\mathcal{S}$ and $\mathcal{A}$, respectively. We consider a finite time horizon $\left[0, T\right]$, where $T < \infty $ is fixed throughout. At time $t$, the environment provides the current state $X_{t}\in\mathcal{S}$. The agent then selects an action $A_{t}=\pi(X_{t})\in\mathcal{A}$ according to policy $\pi$. A reward $R_{t}=r(t,X_{t},A_{t})$ is calculated based on this action. Subsequently, the environment generates the next state $X_{t+1}$ based on the previous state $X_{t}$ and action $A_{t}$. In the underlying true model, the state $X_{t}$ can be described by the following SDE, defined on a complete filtered probability space $(\Omega,\mathcal{F},\mathbb{P},\{\mathcal{F}_{t}\}_{t\geq 0})$:
\begin{align*}
    dX_{t}=b(t,X_{t},A_{t})dt+\sigma(t,X_{t},A_{t})dW_{t}.
\end{align*}
However, in practice, the observed data often contain noise introduced by jumps. This can be modeled by an SDE with jumps as follows:
\begin{align*}
    dX_{t}=b(t,X_{t},A_{t})dt+\sigma(t,X_{t},A_{t})dW_{t}+L(t,X_{t},A_{t})dN_{t}.
\end{align*}

 In real-world scenarios, practical constraints often limit the ability to frequently change policies. For instance, in financial trading, risk management policies may dictate specific trading strategies that cannot be altered frequently. Therefore, this paper focuses on an RL problem with a fixed policy under the continuous-time settings. Since the policy $\pi$ is fixed, $A_{t}=\pi(X_{t})$ is a function of $X_{t}$, we use the following SDE without the action $A_{t}$ to describe the state $X_{t}$:
\begin{align}\label{sde_jump}
    dX_{t}=b(t,X_{t})dt+\sigma(t,X_{t})dW_{t}+L(t,X_{t})dN_{t},
\end{align}
where $W_{t}$ is a $m$-dimensional $\mathcal{F}_{t}$-adapted Brownian motion, $N_{t}$ is a $n$-dimensional $\mathcal{F}_{t}$-adapted counting process independent of $W_{t}$, and $b:[0,T]\times\mathbb{R}^{d}\rightarrow\mathbb{R}^{d}$, $\sigma:[0,T]\times\mathbb{R}^{d}\rightarrow\mathbb{R}^{d\times m}$, and $L:[0,T]\times\mathbb{R}^{d}\rightarrow\mathbb{R}^{d\times n}$ are the drift coefficient, diffusion coefficient, and jump size, respectively.
The value function $J:[0,T]\times\mathbb{R}^{d}\rightarrow\mathbb{R}^{k}$ is defined as:
\begin{align}\label{value}
    J(t,x)=E\left[\int_{t}^{T}r(s,X_{s})ds+h(X_{T})|X_{t}=x\right],
\end{align}
where $r:[0,T]\times\mathbb{R}^{d}\rightarrow\mathbb{R}^{k}$ is a reward function based on the current time and state, and $h:\mathbb{R}^{d}\rightarrow\mathbb{R}^{k}$ is a reward function applied at the end of $T$.

 We make the following standard assumptions.

\begin{assumption}\label{ass1}
 The following conditions hold true.
\begin{itemize}
\item[(i)] $b,\sigma,L,r$, and $h$ are all continuous functions in their respective arguments;

\item[(ii)] $b$ and $\sigma$ are uniformly Lipschitz in $x$, i.e., there exists a constant $C>0$ such that
\[| b(t,x)-b(t,x' )| +| \sigma(t,x)-\sigma(t,x')| \leq C| x-x'| ,\ \forall t\in [0,T],\ x,x'\in\mathbb{R}^{d};\]

\item[(iii)] $b$ and $\sigma$ both have linear growth in $x$, i.e., there exists a constant $C>0$  such that
\[| b(t,x)| ^{2}+| \sigma(t,x)| ^{2}\leq C^{2}(1+| x| ^{2}),\ \forall (t,x)\in[0,T]\times\mathbb{R}^{d};\]

\item[(iv)] $r$, $h$, and $L$ have polynomial growth in $x$, i.e., there exist constants $C>0$ and $\nu\geq 1$ such that for all $ (t,x)\in [0,T]\times\mathbb{R}^{d}$,
\[| r(t,x)| \leq C(1+| x| ^{\nu}),\quad | h(x)| \leq C(1+| x| ^{\nu}),\quad | L(t,x)| \leq C(1+| x| ^{\nu});\]

\item[(v)] The total number of jumps $N_{T}$ is finite with probability 1, i.e. $P(N_{T}<\infty)=1$. Moreover, for all $ i = 0, \dots, n-1 $, for any time interval $[t_{i},t_{i+1}]$ with width $\Delta t=t_{i+1}-t_{i}>0$, the probability of there existing exactly one jump within this interval tends to zero as $\Delta t\rightarrow0$, i.e. $P(N_{t_{i+1}}-N_{t_{i}}=1)\rightarrow0$ as $\Delta t\rightarrow 0$.
\end{itemize}
\end{assumption}

 Assumption \ref{ass1}(i)-(iii) guarantee the existence and uniqueness of a solution to the SDE in (\ref{sde_jump}) (see Theorem 5.2.9 in \cite{karatzas2014brownian}).  Assumption \ref{ass1}(iv) ensures that $J(t,x)$ is finite for any $(t,x)$. Assumption \ref{ass1}(v) ensures that the jump process has finite jumps in finite time.

 Recall that J can be characterized by the following PDE based on the Feynman-Kac formula (\cite{karatzas2014brownian}):
\begin{align}
\left\{\begin{array}{ll}\label{pde}
     \mathcal{L}J(t,x)+r(t,x)=0,\ (t,x)\in [0,T]\times\mathbb{R}^{d}, \\
     J(T,x)=h(x),
\end{array}\right.\end{align}
where
\[\mathcal{L}J(t,x):=\frac{\partial J}{\partial t}(t,x)+b(t,x)\cdot\frac{\partial J}{\partial x}(t,x)+\frac{1}{2}\sigma^{2}(t,x)\cdot\frac{\partial^{2} J}{\partial x^{2}}(t,x).\]

\begin{assumption} \label{ass2}
 The above PDE admits a classical solution $J\in C^{1,2}([0,T]\times\mathbb{R}^{d})$ satisfying the polynomial growth condition, i.e., there exist constants $C>0$ and $\mu\geq 1$ such that
\[|J(t,x)|\leq C(1+|x|^{\mu}),\ \forall (t,x)\in [0,T]\times\mathbb{R}^{d}.\]
\end{assumption}

 Assumption \ref{ass2} guarantees the existence of a solution to the PDE in (\ref{pde}). 

 To analyze this RL problem, we discretize the time interval $[0,T]$ into $[t_{i},t_{i+1}]$ with equal time step $\Delta t$, where $i=0,1,2,\ldots,n-1$ and $0=t_{0}<t_{1}<\cdots<t_{n}=T$. The goal is to estimate the value function $J$ using a parametric family of functions $J^{\theta}$ with $\theta\in\Theta\subset\mathbb{R}^{l}, l\in\mathbb{N}^{+}$ (which could be reduced in complexity if we have some prior knowledge about the value function or the underlying dynamics) based on the observed states $X_{t_{i}}$ and rewards $R_{t_{i}}=r(t_{i},X_{t_{i}})$, without the knowledge of the explicit functional form of $b,\sigma,L,r$, and $h$. A key challenge is to develop an algorithm that remains robust to jumps introduced by noisy observed data, even though the true underlying model follows an SDE without jumps. We need the following assumption for $J^{\theta}$. 
 First, introduce
\begin{align*}
    L_{\mathcal{F}}^{2}([0,T])=&\Big\{\kappa=\{\kappa_{t},0\leq t\leq T\}\text{ is real-valued and }\mathcal{F}_{t}\text{-progressively measurable:}\\
&  \quad E\int_{0}^{T}|\kappa_{t}|^{2}dt<\infty\Big\}.
\end{align*}
Define the $L^{2}$-norm $\Vert\kappa\Vert_{L^{2}}=(E\int_{0}^{T}|\kappa_{t}|^{2}dt)^{1/2}$ on the space $L_{\mathcal{F}}^{2}([0,T])$.
\begin{assumption}\label{ass3}
  $J^{\theta}(t,x)$ is a smooth function of $(t,x,\theta)$ with all the derivatives existing. $E|J^{\theta}(t,X_{t})|^{2}$ is a continuous function of $\theta$ for all $t\in[0,T]$ and $\theta\in\Theta$. Moreover, for all $\theta\in\Theta$, $\mathcal{L}J^{\theta}(\cdot,X_{.}), |\frac{\partial J^{\theta}}{\partial x}(\cdot,X_{.})^{\dagger}\sigma(\cdot,X_{.})|\in L_{\mathcal{F}}^{2}([0,T])$, and their $L^{2}$-norms are continuous functions of $\theta$.
\end{assumption}

 Assumption \ref{ass3} ensures that $E|J^{\theta}(t,X_{t})|^{2}$, $\Vert\mathcal{L}J^{\theta}(\cdot,X_{\cdot})\Vert_{L^{2}}$, and $\Vert\frac{\partial J^{\theta}}{\partial x}(\cdot,X_{.})^{\dagger}\sigma(\cdot,X_{.})\Vert_{L^{2}}$ are finite in any given compact set $\Theta$.

\section{Theoretical Analysis of Algorithms} \label{sec-3}
We commence our exploration by revisiting the Temporal-Difference (TD) Error proposed by \cite{Doya:2000} and the Mean-Square TD Error (MSTDE) algorithm, along with its convergence behavior described by \cite{jiaandzhou2022}.

\subsection{Deterministic Setting}
In the deterministic setting, the state $X_{t}$ follows an ODE:
\[dX_{t}=b(t,X_{t})dt.\]
Since the states $X_{t}$ are deterministic, the expectation equals itself, and the equation (\ref{value}) becomes 
\[J(t,X_{t})=\int_{t}^{T}r(s,X_{s})ds+h(X_{T})=J(t+\Delta t,X_{t+\Delta t})+\int_{t}^{t+\Delta t}r(s,X_{s})ds.\]
Moving $J(t,X_{t})$ to the right side and dividing both sides by $\Delta t$, we have 
\begin{align}\label{tde}
    \frac{J(t+\Delta t,X_{t+\Delta t})-J(t,X_{t})}{\Delta t}+\frac{1}{\Delta t}\int_{t}^{t+\Delta t}r(s,X_{s})ds=0.
\end{align}
As $\Delta t\rightarrow 0$, the left side of (\ref{tde}) becomes 
\begin{align*}
    \delta_{t}:=\frac{d}{dt}J(t,X_{t})+r(t,X_{t}),
\end{align*}
which is the the well-known TD Error. When estimating the value function $J$ using $J^{\theta}$ with parameter $\theta$, the TD Error becomes  
\[\delta_{t}^{\theta}:=\frac{\partial}{\partial t}J^{\theta}(t,X_{t})+r(t,X_{t}).\]
Building upon the above TD Error, \cite{Doya:2000} introduced the mean-square TD error (MSTDE):
\[\text{MSTDE}(\theta):=\frac{1}{2}\int_{0}^{T}\left|\delta_{t}^{\theta}\right|^{2}dt=\frac{1}{2}\int_{0}^{T}\left|\frac{\partial}{\partial t}J^{\theta}(t,X_{t})+r(t,X_{t})\right|^{2}dt.\]
The optimal $\theta$ is determined by minimizing the above MSTDE, since if $J^{\theta}$ is a perfect estimator of the true value function $J$, then $\delta_{t}^{\theta}$ would be zero. To solve the RL problem using the observed discretized data $X_{t_{i}}$ and $r(t_{i},X_{t_{i}})$, \cite{jiaandzhou2022} defined the $\text{MSTDE}_{\Delta t}(\theta)$ which is an approximation of $\text{MSTDE}(\theta)$:
\[\text{MSTDE}_{\Delta t}(\theta):=\frac{1}{2}\sum\limits_{i=0}^{n-1}\left[\frac{J^{\theta}(t_{i+1},X_{t_{i+1}})-J^{\theta}(t_{i},X_{t_{i}})}{t_{i+1}-t_{i}}+r(t_{i},X_{t_{i}})\right]^{2}\Delta t,\]
and proposed an algorithm to update $\theta$ by minimizing the $\text{MSTDE}_{\Delta t}(\theta)$ using the gradient descent method.

\subsection{Stochastic Setting with Continuous Dynamics} \label{sec3.2}

The state $X_{t}$ follows a continuous SDE:
\begin{align}\label{sde}
    dX_{t}=b(t,X_{t})dt+\sigma(t,X_{t})dW_{t}.
\end{align}
We first recall the findings of \cite{elkaroui1997}, the PDE in (\ref{pde}) is connected to the following forward-backward stochastic differential equation (FBSDE):
\begin{align}
\left\{\begin{array}{ll}\label{fbsde}
     dX_{s}=b(s,X_{s})ds+\sigma(s,X_{s})dW_{s},\ s\in[t,T];\ X_{t}=x, \\
     dY_{s}=-r(s,X_{s})ds+Z_{s}dW_{s},\ s\in[t,T];\ Y_{T}=h(X_{T}).
\end{array}\right.\end{align}
The value function $J(s,X_{s})$ and the solution $\{(X_{s},Y_{s},Z_{s}),\ t\leq s\leq T\}$ of the above FBSDE have the following relationship:
\begin{align}
\left\{\begin{array}{ll}\label{sol}
     Y_{s}=J(s,X_{s}),\ s\in[t,T], \\
     Z_{s}=\frac{\partial J}{\partial x}(s,X_{s})^{\dagger}\sigma(s,X_{s}),\ s\in[t,T].
\end{array}\right.\end{align}
For any fixed $(t,x)\in[0,T]\times\mathbb{R}^{d}$ and $\{X_{s},\ t\leq s\leq T\}$ solving the first equation of (\ref{fbsde}), define
\begin{align}\label{mart}
    M_{s}:=J(s,X_{s})+\int_{t}^{s}r(u,X_{u})du,\ s\in [t,T].
\end{align}
By combining (\ref{fbsde}) and (\ref{sol}), we observe the following result:
\begin{align*}%\label{dm}
    dM_{s}&=dJ(s,X_{s})+r(s,X_{s})ds\\
    & = dY_{s}+r(s,X_{s})ds\\
    & = Z_{s}dW_{s}.
\end{align*}
Assumption \ref{ass1} guarantees that $E[\int_{0}^{T}Z_{s}^{2}ds]<\infty$, which implies $M_{s}$ is a martingale. The martingale property implies that for any $\Delta t>0$, 
\begin{align*}%\label{martprop}
    E[M_{t+\Delta t}|\mathcal{F}_{t}]=M_{t},
\end{align*}
and by plugging in the formula for $M$, we have 
\begin{align*}
    E\left[J(t+\Delta t,X_{t+\Delta t})+\int_{t}^{t+\Delta t}r(s,X_{s})ds\Big|\mathcal{F}_{t}\right]=J(t,X_{t}).
\end{align*}
Similar to (\ref{tde}), we obtain 
\[E\left[\frac{J(t+\Delta t,X_{t+\Delta t})-J(t,X_{t})}{\Delta t}+\frac{1}{\Delta t}\int_{t}^{t+\Delta t}r(s,X_{s})ds\right]=0.\]
As $\Delta t$ goes to zero, $\frac{d}{dt}J(t,X_{t})$ does not exist due to the non-differentiability of Brownian motion. Although the MSTDE does not theoretically exist in this setting, we can still define $\text{MSTDE}_{\Delta t}(\theta)$ similar to the deterministic case as follows:
\[\text{MSTDE}_{\Delta t}(\theta):=\frac{1}{2}E\left[\sum\limits_{i=0}^{n-1}\left(\frac{J^{\theta}(t_{i+1},X_{t_{i+1}})-J^{\theta}(t_{i},X_{t_{i}})}{t_{i+1}-t_{i}}+r(t_{i},X_{t_{i}})\right)^{2}\Delta t\right].\]
Similar to (\ref{mart}), we can define 
\[M^{\theta}_{s}:=J^{\theta}(s,X_{s})+\int_{t}^{s}r(u,X_{u})du.\]
Then, we have
\begin{align*}
    \text{MSTDE}_{\Delta t}(\theta)&=\frac{1}{2\Delta t}E\left[\sum\limits_{i=0}^{n-1}(J^{\theta}(t_{i+1},X_{t_{i+1}})-J^{\theta}(t_{i},X_{t_{i}})+r(t_{i},X_{t_{i}})\Delta t)^{2}\right]\\
    & \approx \frac{1}{2\Delta t}E\left[\sum\limits_{i=0}^{n-1}(J^{\theta}(t_{i+1},X_{t_{i+1}})-J^{\theta}(t_{i},X_{t_{i}})+\int_{t_{i}}^{t_{i+1}}r(s,X_{s})ds)^{2}\right]\\
    & = \frac{1}{2\Delta t}E\left[\sum\limits_{i=0}^{n-1}(M^{\theta}_{t_{i+1}}-M^{\theta}_{t_{i}})^{2}\right]\\
    & \rightarrow\frac{1}{2\Delta t}E\langle M^{\theta}\rangle_{T}.
\end{align*}
The above limit holds as the size of the partition goes to 0. This provides the intuition that minimizing the $\text{MSTDE}_{\Delta t}(\theta)$ will guide the optimal parameter $\theta^{*}$ towards minimizing the expected quadratic variation of $M^{\theta}$. By It\^o's formula, we have 
\begin{align*}
    dM_{t}^{\theta}=\left[\mathcal{L}J^{\theta}(t,X_{t})+r(t,X_{t})\right]dt+\left(\frac{\partial J^{\theta}(t,X_{t})}{\partial x}\right)^{\dagger}\sigma(t,X_{t})dW_{t}.
\end{align*}
Hence, we have
\begin{align*}
    d\langle M^{\theta}\rangle_{t}&\equiv(dM_{t}^{\theta})^{2}=\left\{\left[\mathcal{L}J^{\theta}(t,X_{t})+r(t,X_{t})\right]dt+\left(\frac{\partial J^{\theta}(t,X_{t})}{\partial x}\right)^{\dagger}\sigma(t,X_{t})dW_{t}\right\}^{2}\\
    & = \left[\left(\frac{\partial J^{\theta}(t,X_{t})}{\partial x}\right)^{\dagger}\sigma(t,X_{t})\right]^{2}(dW_{t})^{2}+\left[\mathcal{L}J^{\theta}(t,X_{t})+r(t,X_{t})\right]^{2}(dt)^{2}\\
    &\quad +2\left[\mathcal{L}J^{\theta}(t,X_{t})+r(t,X_{t})\right]dt\cdot\left(\frac{\partial J^{\theta}(t,X_{t})}{\partial x}\right)^{\dagger}\sigma(t,X_{t})dW_{t}\\
    & = \left[\left(\frac{\partial J^{\theta}(t,X_{t})}{\partial x}\right)^{\dagger}\sigma(t,X_{t})\right]^{2}dt+\text{ high-order small term},
\end{align*}
where the last line is due to $(dW_{t})^{2}=dt$ a.s. and the fact that the terms involving $(dt)^{2}$ and $dW_{t}dt$ are of higher order than $dt$. Therefore, minimizing the expected quadratic variation of $M^{\theta}$ is equivalent to minimizing $E\left[\int_{0}^{T}\left|\left(\frac{\partial J^{\theta}(t,X_{t})}{\partial x}\right)^{\dagger}\sigma(t,X_{t})\right|^{2}dt\right]$. Since the reward function $r(t_{i},X_{t_{i}})$ actually does not contribute to the quadratic variation of $M^{\theta}$ as seen before, we use the following error term instead of $\text{MSTDE}_{\Delta t}(\theta)$:
\begin{align}\label{err}
    \text{MSTDE}_{\Delta t}^{*}(\theta):=E\left[\sum\limits_{i=0}^{n-1}\left(J^{\theta}(t_{i+1},X_{t_{i+1}})-J^{\theta}(t_{i},X_{t_{i}})\right)^{2}\right].
\end{align}

 \cite{jiaandzhou2022} proved the following result (Theorem 2 in \cite{jiaandzhou2022}):

 Suppose the state $X_{t}$ follows the equation (\ref{sde}) and Assumptions \ref{ass1}-\ref{ass3} hold. Let $\theta_{\text{MSTDE}^{*}}^{\star}(\Delta t)\in\arg\min\limits_{\theta\in\Theta}\text{MSTDE}_{\Delta t}^{*}(\theta)$, and assume that $\theta_{\text{MSTDE}^{*}}^{\star}:=\lim_{\Delta t\rightarrow 0}\theta_{\text{MSTDE}^{*}}^{\star}(\Delta t)$ exists. Then, we have 
\[\theta_{\text{MSTDE}^{*}}^{\star}\in\arg\min\limits_{\theta\in\Theta}E\left[\int_{0}^{T}\left|\left(\frac{\partial J^{\theta}(t,X_{t})}{\partial x}\right)^{\dagger}\sigma(t,X_{t})\right|^{2}dt\right].\]

\subsection{Stochastic Setting with Jump Noise}
 In practice, jumps are frequently observed and are typically triggered by unexpected news or events. These jumps exhibit a fundamentally different nature from the continuous fluctuations captured by continuous diffusion processes. As a result, the presence of jumps complicates the task of learning the dynamics driven by the latent continuous diffusion process in (\ref{sde}). To address this, we assume that the true underlying state evolves according to a diffusion process without jumps, while the observed data are contaminated by an additive jump component as in  \eqref{sde_jump}. Our objective is to develop a robust algorithm that can effectively account for and mitigate the impact of jump noise in the observed process.
  For example, when $X_{t}$ represents a stock price, we want to learn dynamics from the rate of the return, represented by $\int_{0}^{T}b(t,X_{t})dt$, and from the random fluctuations due to the market volatility, represented by $\int_{0}^{T}\sigma(t,X_{t})dW_{t}$, rather than to the jump components, represented by $\int_{0}^{T}L(t,X_{t})dN_{t}$. 
However,  in the presence of jumps, the MSTDE algorithm converges to the expected quadratic variation of $M^{\theta}$, which differs from the target quantity of interest.
In the following theorem, we formally characterize the bias introduced by the MSTDE algorithm under jump contamination.

\begin{theorem}\label{thm1}
 Suppose the state $X_{t}$ follows the equation (\ref{sde_jump}) and Assumptions \ref{ass1}--\ref{ass3} hold. Let 
\[\theta_{\text{MSTDE}^{*}}^{\star}(\Delta t)\in\arg\min\limits_{\theta\in\Theta}\text{MSTDE}_{\Delta t}^{*}(\theta),\]
and assume that $\theta_{\text{MSTDE}^{*}}^{\star}:=\lim_{\Delta t\rightarrow 0}\theta_{\text{MSTDE}^{*}}^{\star}(\Delta t)$ exists. 
Then, we have
\begin{align}\label{mstde_conv}
\theta_{\text{MSTDE}^{*}}^{\star}\in\arg\min\limits_{\theta\in\Theta} &E\Bigg[\int_{0}^{T}\left|\left(\frac{\partial J^{\theta}(t,X_{t-})}{\partial x}\right)^{\dagger}\sigma(t,X_{t-})\right|^{2}dt \cr
	&\qquad \qquad \qquad   +\int_{0}^{T}\left[J^{\theta}(t,X_{t})-J^{\theta}(t,X_{t-})\right]^{2}dN_{t}\Bigg],
\end{align}
where $X_{t-}:=\lim_{\Delta t\searrow 0}X_{t-\Delta t}$ is the left limit of $X_{t}$.
\end{theorem}

This result reflects the setting where we only observe the jump-contaminated process $X_t$, as given by equation \eqref{sde_jump}, rather than the latent continuous path.
Theorem \ref{thm1} reveals that the MSTDE algorithm suffers from two sources of bias: the jump variation term,  $\int_{0}^{T}\left[J^{\theta}(t,X_{t})-J^{\theta}(t,X_{t-})\right]^{2}dN_{t}$, and the distortion in the gradient evaluation due to discontinuities in the observed process, $\frac{\partial J^{\theta}(t,X_{t-})}{\partial x}$.

 To address the issue posed by the presence of jump components in the state dynamics, we introduce the Mean-Square Bipower Variation Error (MSBVE) as an alternative error term. The MSBVE is defined as follows:
\begin{align}\label{msbve_err}
    \text{MSBVE}_{\Delta t}(\theta):=E\left[\sum\limits_{i=1}^{n-1}\left|J^{\theta}(t_{i+1},X_{t_{i+1}})-J^{\theta}(t_{i},X_{t_{i}})\right|\left|J^{\theta}(t_{i},X_{t_{i}})-J^{\theta}(t_{i-1},X_{t_{i-1}})\right|\right].
\end{align}
Unlike MSTDE, which considers the squared difference between consecutive value function estimates, the MSBVE approach utilizes the product of non-overlapping absolute differences. Since jumps are rare events, there is at most one jump within a short time interval. By leveraging this property, the use of non-overlapping absolute differences helps eliminate the influence of jump components. As a result, MSBVE is capable of mitigating the impact of jump noise in the observed process. The following theorem formally demonstrates the advantage of the MSBVE approach in jump-contaminated environments.

\begin{theorem}\label{thm2}
Suppose the state $X_{t}$ follows the equation (\ref{sde_jump}) and Assumptions \ref{ass1}--\ref{ass3} hold. Let 
\[\theta_{\text{MSBVE}}^{\star}(\Delta t)\in\arg\min\limits_{\theta\in\Theta}\text{MSBVE}_{\Delta t}(\theta),\]
and assume that $\theta_{\text{MSBVE}}^{\star}:=\lim_{\Delta t\rightarrow 0}\theta_{\text{MSBVE}}^{\star}(\Delta t)$ exists. Then, we have
\[\theta_{\text{MSBVE}}^{\star}\in\arg\min\limits_{\theta\in\Theta}E\left[\int_{0}^{T}\left|\left(\frac{\partial J^{\theta}(t,X_{t-})}{\partial x}\right)^{\dagger}\sigma(t,X_{t-})\right|^{2}dt\right].\]
\end{theorem}

Theorem \ref{thm2} shows that the proposed MSBVE estimator $\theta_{\text{MSBVE}}^{\star}$ is free from the bias associated with the jump variation term,  $\int_{0}^{T}\left[J^{\theta}(t,X_{t})-J^{\theta}(t,X_{t-})\right]^{2}dN_{t}$. 
In other words, compared to the MSTDE approach, the MSBVE approach exhibits greater robustness to the presence of jump components.
The effectiveness of MSBVE will be further illustrated through numerical studies presented in Sections \ref{simulation} and \ref{application}.
In the next section, we analyze the remaining bias arising from the gradient term of the value function, $\frac{\partial J^{\theta}(t,X_{t-})}{\partial x}$.

\subsection{Bias Analysis}\label{sec3.4}
We can decompose the observed jump-diffusion process $X_{t}$ as follows:
\[dX_{t}=dX_{t}^{C}+dX_{t}^{J},\]
where $dX_{t}^{C}=b(t,X_{t})dt+\sigma(t,X_{t})dW_{t}$ and $dX_{t}^{J}=L(t,X_{t})dN_{t}$. That is, $X_{t}^{C}$ and $X_{t}^{J}$ represent the continuous and jump components, respectively. 
Based on the latent continuous diffusion process, as discussed in Section \ref{sec3.2}, the oracle estimator
is defined as follows:
\[\theta_{\text{Oracle}}^{\star}\in\arg\min\limits_{\theta\in\Theta}E\left[\int_{0}^{T}\left|\left(\frac{\partial J^{\theta}(t,X_{t}^{C})}{\partial x}\right)^{\dagger}\sigma(t,X_{t-})\right|^{2}dt\right].\]
The proposed MSBVE estimation procedure still has a bias compared with the
oracle estimator. The bias is coming from the difference between $J^{\theta}(t,X_{t-})$ and $J^{\theta}(t,X_{t}^{C})$. More specifically, for the general vector case, by Taylor expansion, we have  
\[
\nabla_x J^\theta(t, X_t)^\top \nabla_x J^\theta(t, X_t) = \nabla_x J^\theta(t, X_t^C)^\top \nabla_x J^\theta(t, X_t^C) + 2 \nabla_x J^\theta(t, X_t^C)^\top H_x J^\theta(t, X_t^C) X_t^J 
\]  
\[
+ \left( \text{tr}(H_x J^\theta(t, X_t^C) H_x J^\theta(t, X_t^C)) + \nabla_x J^\theta(t, X_t^C)^\top \nabla_x (\text{tr}(H_x J^\theta(t, X_t^C))) \right) \| X_t^J \|^2 + \cdots,
\]  
where, for any differentiable function \( g: \mathbb{R}^d \to \mathbb{R} \),  
\[
\nabla_x g(t, x_0) = \frac{\partial g(t, x)}{\partial x} \Bigg|_{x = x_0}
\]  
is the gradient, and  
\[
H_x g(t, x_0) = \frac{\partial^2 g(t, x)}{\partial x \partial x^\top} \Bigg|_{x = x_0}
\]  
is the Hessian matrix. 
Then, the bias comes from the series:  
\[
\text{bias}(\theta):=\sum_{i=1}^{\infty} \frac{f^{(i)}(X_t^C)}{i!} (X_t^J)^{\otimes i},
\]  
where \( f^{(i)}(X_t^C) \) represents the \( i \)th derivative tensor of \( \nabla_x J^\theta(t, X_t^C)^\top \nabla_x J^\theta(t, X_t^C) \) with respect to \( x \), and \( (X_t^J)^{\otimes i} \) denotes the \( i \)th order tensor product of \( X_t^J \).  
Thus, the solution of the MSBVE procedure is 
\[\theta_{\text{MSBVE}}^{\star}\in\arg\min\limits_{\theta\in\Theta}E\left[\int_{0}^{T}\left|\left(\frac{\partial J^{\theta}(t,X_{t}^{C})}{\partial x}\right)^{\dagger}\sigma(t,X_{t-})\right|^{2}dt\right]+\text{bias}(\theta),\]
where the bias is a multivariate Taylor’s expansion. From this result, we can find that
when the value function is linear in $x$, the bias term bias($\theta$) disappears. Generally, when
the value function is polynomial in $x$, higher-order terms remain. We will examine how this bias behaves in the following simulation and application sections.

\section{Simulation Study} \label{simulation}
In our simulation setup, we consider a state process $X_{t}$, defined on the interval $t\in[0,1]$, which evolves according to the SDE:
$$dX_{t}=dW_{t}+X_{t-}dN_{t},$$
where $W_{t}$ is a standard Brownian Motion and $N_{t}$ is a Poisson process independent of $W_{t}$. We initialize the process with $X_{0}=0.1$. We only select those Poisson process paths where exactly one jump occurs in $[0,1]$, which implies the jump occurs uniformly within the interval $[0,1]$. Under this setting, we have
\begin{align*}
    X_{t}-X_{t-}=\begin{cases}
    X_{t-}, &\text{if }\Delta N_{t}=1,\\
    0, &\text{if }\Delta N_{t}=0.
    \end{cases}
\end{align*}
Note that under this model, the jump size at time $t$ is equal to $X_{t-}$, meaning the process doubles at the jump time: $X_t = 2 X_{t-}$. Prior to the jump, however,  $X_{t}$ evolves a Brownian Motion starting from $X_{0}=0.1$. We compare the performance of the algorithms under three different forms of the value function:
\begin{itemize}
    \item [(i)] Linear:\[
J^{\theta}(t,x) = (\theta(1-t) + 1)x;
\]
    \item [(ii)] Quadratic:\[
J^{\theta}(t,x) = \theta(1-t)x^2+x;
\]
    \item [(iii)] Exponential:\[
J^{\theta}(t,x) = \theta(1-t)e^x+x.
\]
\end{itemize}

For the MSBVE algorithm, we utilize stochastic gradient descent (SGD) to update the parameter $\theta$ using the error term defined in (\ref{msbve_err}):
\begin{align*}
\text{MSBVE}_{\Delta t}(\theta):=E\Big[\sum\limits_{i=1}^{n-1} |J^{\theta}_{i+1}-J^{\theta}_{i}|\cdot|J^{\theta}_{i}-J^{\theta}_{i-1}|\Big],
\end{align*}
where 
\[J^{\theta}_{i}=J^{\theta}(t_{i},X_{t_{i}}).\]
The update rule for $\theta$ following from SGD is 
\begin{align*}
    \theta\leftarrow\theta-\alpha\times\sum\limits_{i=1}^{n-1} & \Big[\frac{\partial}{\partial\theta}(J^{\theta}_{i+1}-J^{\theta}_{i})|J^{\theta}_{i}-J^{\theta}_{i-1}|\text{sgn}(J^{\theta}_{i+1}-J^{\theta}_{i})\\
    & +\frac{\partial}{\partial\theta}(J^{\theta}_{i}-J^{\theta}_{i-1})|J^{\theta}_{i+1}-J^{\theta}_{i}|\text{sgn}(J^{\theta}_{i}-J^{\theta}_{i-1})\Big].
\end{align*}
For the MSTDE algorithm, the error term in (\ref{err}) becomes
\[\text{MSTDE}_{\Delta t}^{*}(\theta):=E\left[\sum\limits_{i=0}^{n-1}\left(J^{\theta}_{i+1}-J^{\theta}_{i}\right)^{2}\right],\]
and the update rule for $\theta$ follows from SGD becomes 
\begin{align*}
    \theta\leftarrow\theta-2\alpha\times\sum\limits_{i=0}^{n-1}\left(J^{\theta}_{i+1}-J^{\theta}_{i}\right)\cdot\frac{\partial}{\partial\theta}\left(J^{\theta}_{i+1}-J^{\theta}_{i}\right).
\end{align*}

\subsection{Linear Value Function}
When $J^{\theta}(t,x) = (\theta(1-t) + 1)x$,
based on the discussion of Theorem \ref{thm2}, the theoretical minimizer of MSBVE is
\begin{align*}
    &\theta_{\text{MSBVE}}^{*}  \in\arg\min_{\theta\in\Theta}E\left[\int_{0}^{1}\left|\left(\frac{\partial J^{\theta}(t,X_{t-})}{\partial x}\right)^{\dagger}\sigma(t,X_{t-})\right|^{2}dt\right]\\
    & = \arg\min_{\theta\in\Theta}\int_{0}^{1}\left[(\theta(1-t)+1)\right]^{2}dt\\
    & = \arg\min_{\theta\in\Theta}\left(\frac{1}{3}\theta^{2}+\theta+1\right)\\
    & = -\frac{3}{2},
\end{align*}
and based on the discussion of Theorem \ref{thm1}, the theoretical minimizer of MSTDE is
\begin{align*}
    &\theta_{\text{MSTDE}}^{*}  \in\arg\min_{\theta\in\Theta}E\left[\int_{0}^{1}\left|\left(\frac{\partial J^{\theta}(t,X_{t-})}{\partial x}\right)^{\dagger}\sigma(t,X_{t-})\right|^{2}dt+\int_{0}^{1}\left[J^{\theta}(t,X_{t})-J^{\theta}(t,X_{t-})\right]^{2}dN_{t}\right]\\
    & = \arg\min_{\theta\in\Theta}E\left[\int_{0}^{1}\left[(\theta(1-t)+1)\right]^{2}dt+\int_{0}^{1}\left[(\theta(1-t)+1)(X_{t}-X_{t-})\right]^{2}dN_{t}\right]\\
    & = \arg\min_{\theta\in\Theta}\left\{\frac{1}{3}\theta^{2}+\theta+1\right.\\
    &\left. \qquad \qquad +E\left[E\left(\int_{0}^{1}\left[(\theta(1-t)+1)(X_{t}-X_{t-})\right]^{2}dN_{t}\bigg |\text{ the jump happens at time } u\right)\right]\right\}\\
    & = \arg\min_{\theta\in\Theta}\left[\frac{1}{3}\theta^{2}+\theta+1 + E_{U\sim\text{Unif}[0,1]}\left((\theta(1-U)+1)^{2}E(X_{U-}^{2}|\text{ the jump happens at time }u)\right)\right]\\
    & = \arg\min_{\theta\in\Theta}\Bigg[\frac{1}{3}\theta^{2}+\theta+1\\
&\qquad \qquad    +E_{U\sim\text{Unif}[0,1]}\left((\theta(1-U)+1)^{2}E((W_{U}+0.1)^{2}|\text{ the jump happens at time }u)\right)\Bigg]\\
    & = \arg\min_{\theta\in\Theta}\left[\frac{1}{3}\theta^{2}+\theta+1+E_{U\sim\text{Unif}[0,1]}\left((\theta(1-U)+1)^{2}(U+0.01)\right)\right]\\
    & = \arg\min_{\theta\in\Theta}\left(\frac{21}{50}\theta^{2}+\frac{403}{300}\theta+\frac{151}{100}\right)\\
    & = -\frac{403}{252}.
\end{align*}
It is worth noting that, since the value function is linear in $x$, the bias term discussed in Section \ref{sec3.4} vanishes. As a result, the MSBVE estimator achieves the oracle solution: $\theta_{\text{Oracle}}^{*}=\theta_{\text{MSBVE}}^{*}=-\frac{3}{2}$.

\subsection{Quadratic Value Function}
When $J^{\theta}(t,x) = \theta(1-t) x^2 + x$,
based on the discussion of Theorem \ref{thm2}, the theoretical minimizer of MSBVE is
\begin{align*}
    &\theta_{\text{MSBVE}}^{*}  \in\arg\min_{\theta\in\Theta}E\left[\int_{0}^{1}\left|\left(\frac{\partial J^{\theta}(t,X_{t-})}{\partial x}\right)^{\dagger}\sigma(t,X_{t-})\right|^{2}dt\right]\\
    & = \arg\min_{\theta\in\Theta}E\left[\int_{0}^{1}\left[2\theta(1-t)X_{t-}+1\right]^{2}dt\right]\\
    & = \arg\min_{\theta\in\Theta}E\Bigg[\int_{0}^{u}E\left(2\theta(1-t)(W_{t}+0.1)+1\right)^{2}dt \\
	&\qquad \qquad \qquad \qquad +\int_{u}^{1}E\left(2\theta(1-t)(W_{t}+W_{u}+0.2)+1\right)^{2}dt\Bigg |\text{ the jump happens at time }u\Bigg]\\
    & = \int_{0}^{1}\int_{0}^{u}\left(4\theta^{2}(1-t)^{2}(t+0.01)+0.4\theta(1-t)+1\right)dtdu \\
    &\quad\quad+ \int_{0}^{1}\int_{u}^{1}\left(4\theta^{2}(1-t)^{2}(3u+t+0.04)+0.8\theta(1-t)+1\right)dtdu\\
    & = \arg\min_{\theta\in\Theta}\left(\frac{167}{300}\theta^{2}+\frac{4}{15}\theta+1\right)\\
    & = -\frac{40}{167},
\end{align*}
and based on the discussion of Theorem \ref{thm1}, the theoretical minimizer of MSTDE is
\begin{align*}
    &\theta_{\text{MSTDE}}^{*}  \in\arg\min_{\theta\in\Theta}E\left[\int_{0}^{1}\left|\left(\frac{\partial J^{\theta}(t,X_{t-})}{\partial x}\right)^{\dagger}\sigma(t,X_{t-})\right|^{2}dt+\int_{0}^{1}\left[J^{\theta}(t,X_{t})-J^{\theta}(t,X_{t-})\right]^{2}dN_{t}\right]\\
    & = \arg\min_{\theta\in\Theta}\Bigg[\frac{167}{300}\theta^{2}+\frac{4}{15}\theta+1 \\
& \qquad \qquad +E\left[(3\theta(1-u)(W_{u}+0.1)^{2}+(W_{u}+0.1))^{2}\Bigg|\text{ the jump happens at time }u\right]\Bigg]\\
    & = \arg\min_{\theta\in\Theta}\left[\frac{167}{300}\theta^{2}+\frac{4}{15}\theta+1\right.\\
    &\quad \left.+\int_{0}^{1}\left(9\theta^{2}(1-u)^{2}(3u^{2}+0.06u+0.0001)+6\theta(1-u)(0.3u+0.001)+(u+0.01)\right)du\right]\\
    & = \arg\min_{\theta\in\Theta}\left(\frac{45059}{30000}\theta^{2}+\frac{1709}{3000}\theta+\frac{151}{100}\right)\\
    & = -\frac{8545}{45059}.
\end{align*}
For the latent continuous process, $dX_{t}^{C}=dW_{t}$, therefore, the oracle estimator is 
\begin{align*}
    &\theta_{\text{Oracle}}^{*}  \in\arg\min_{\theta\in\Theta}E\left[\int_{0}^{1}\left|\left(\frac{\partial J^{\theta}(t,X_{t}^{C})}{\partial x}\right)^{\dagger}\sigma(t,X_{t-})\right|^{2}dt\right]\\
    & = \arg\min_{\theta\in\Theta}E\left[\int_{0}^{1}\left[2\theta(1-t)(W_{t}+0.1)+1\right]^{2}dt\right]\\
    & = \int_{0}^{1}\left(4\theta^{2}(1-t)^{2}(t+0.01)+0.4\theta(1-t)+1\right)dt \\
    & = \arg\min_{\theta\in\Theta}\left(\frac{26}{75}\theta^{2}+\frac{1}{5}\theta+1\right)\\
    & = -\frac{15}{52}.
\end{align*}
Unlike the linear value function case, here the value function is nonlinear in $x$, so the bias term discussed in Section \ref{sec3.4} does not vanish. As a result, the oracle solution and the MSBVE estimate differ: $\theta_{\text{Oracle}}^{*} \neq \theta_{\text{MSBVE}}^{*}$.

\subsection{Exponential Value Function}
When $J^{\theta}(t,x) = \theta(1-t)e^x+x$,
based on the discussion of Theorem \ref{thm2}, the theoretical minimizer of MSBVE is
\begin{align*}
    &\theta_{\text{MSBVE}}^{*}  \in\arg\min_{\theta\in\Theta}E\left[\int_{0}^{1}\left|\left(\frac{\partial J^{\theta}(t,X_{t-})}{\partial x}\right)^{\dagger}\sigma(t,X_{t-})\right|^{2}dt\right]\\
    & = \arg\min_{\theta\in\Theta}E\left[\int_{0}^{1}\left[\theta(1-t)e^{X_{t-}}+1\right]^{2}dt\right]\\
    & = \arg\min_{\theta\in\Theta}E\Bigg[\int_{0}^{u}E\left(\theta(1-t)e^{W_{t}+0.1}+1\right)^{2}dt\\
&\qquad \qquad \qquad \qquad +\int_{u}^{1}E\left(\theta(1-t)e^{W_{t}+W_{u}+0.2}+1\right)^{2}dt\Bigg |\text{ the jump happens at time }u\Bigg]\\
    & = \arg\min_{\theta\in\Theta}\left[\int_{0}^{1}\int_{0}^{u}\left(\theta^{2}(1-t)^{2}e^{2t+0.2}+2\theta(1-t)e^{\frac{1}{2}t+0.1}+1\right)dtdu\right. \\
    &\left.\quad\quad+ \int_{0}^{1}\int_{u}^{1}\left(\theta^{2}(1-t)^{2}e^{(6u+2t+0.4)}+2\theta(1-t)e^{\frac{1}{2}(3u+t)+0.2}+1\right)dtdu\right]\\
    & = \arg\min_{\theta\in\Theta}\left[3.190\theta^2+1.657\theta+1\right]\approx-0.260, 
\end{align*}
and based on the discussion of Theorem \ref{thm1}, the theoretical minimizer of MSTDE is
\begin{align*}
    &\theta_{\text{MSTDE}}^{*}  \in\arg\min_{\theta\in\Theta}E\left[\int_{0}^{1}\left|\left(\frac{\partial J^{\theta}(t,X_{t-})}{\partial x}\right)^{\dagger}\sigma(t,X_{t-})\right|^{2}dt+\int_{0}^{1}\left[J^{\theta}(t,X_{t})-J^{\theta}(t,X_{t-})\right]^{2}dN_{t}\right]\\
    & = \arg\min_{\theta\in\Theta}\left[\left(3.190\theta^{2}+1.657\theta+1\right)\right.\\
    &\quad\left.+E\left((\theta(1-u)(e^{2W_{u}+0.2}-e^{W_{u}+0.1})+W_{u}+0.1)^{2}\Bigg|\text{ the jump happens at time }u\right)\right]\\
    & = \arg\min_{\theta\in\Theta}\bigg[\left(3.190\theta^{2}+1.657\theta+1\right)\\
    &\quad+\int_{0}^{1}\theta^{2}(1-u)^{2}(e^{8u+0.4}-2e^{\frac{9}{2}u+0.3}+e^{2u+0.2})du\\
    &\quad\left.+\int_{0}^{1}\left(2\theta(1-u)((2u+0.1)e^{2u+0.2}-(u+0.1)e^{\frac{1}{2}u+0.1})+(u+0.01)\right)du\right]\\
    & = \arg\min_{\theta\in\Theta}\left[7.607\theta^2 + 2.965\theta + 1.505\right]\approx -0.195.
\end{align*}
For the latent continuous process, $dX_{t}^{C}=dW_{t}$, therefore, the oracle estimator is 
\begin{align*}
    &\theta_{\text{Oracle}}^{*}  \in\arg\min_{\theta\in\Theta}E\left[\int_{0}^{1}\left|\left(\frac{\partial J^{\theta}(t,X_{t}^{C})}{\partial x}\right)^{\dagger}\sigma(t,X_{t-})\right|^{2}dt\right]\\
    & = \arg\min_{\theta\in\Theta}E\left[\int_{0}^{1}\left[\theta(1-t)e^{W_{t}+0.1}+1\right]^{2}dt\right]\\
    & = \arg\min_{\theta\in\Theta}\int_{0}^{1}\left(\theta^{2}(1-t)^{2}e^{2t+0.2}+2\theta(1-t)e^{\frac{1}{2}t+0.1}+1\right)dt \\
    & = -\frac{\int_{0}^{1}(1-t)e^{\frac{1}{2}t+0.1}dt}{\int_{0}^{1}(1-t)^{2}e^{2t+0.2}dt}\\
    & \approx-0.901.
\end{align*}

As in the quadratic case, the value function is nonlinear in $x$, so the bias introduced by jump contamination remains. Consequently, the MSBVE estimator does not recover the oracle solution: $\theta_{\text{Oracle}}^{*} \neq \theta_{\text{MSBVE}}^{*}$.

\subsection{Results}
In our simulation study, we set the time grid size as $\Delta t=0.001$, learning rate $\alpha=0.0005$, and initialize $\theta_{0}=0.5$. We train the model for 100000 episodes with 32 sample paths generated for each episode.

\begin{figure}[htbp]
    \centering
    \begin{subfigure}[b]{0.32\textwidth}
        \centering
        \includegraphics[width=\textwidth]{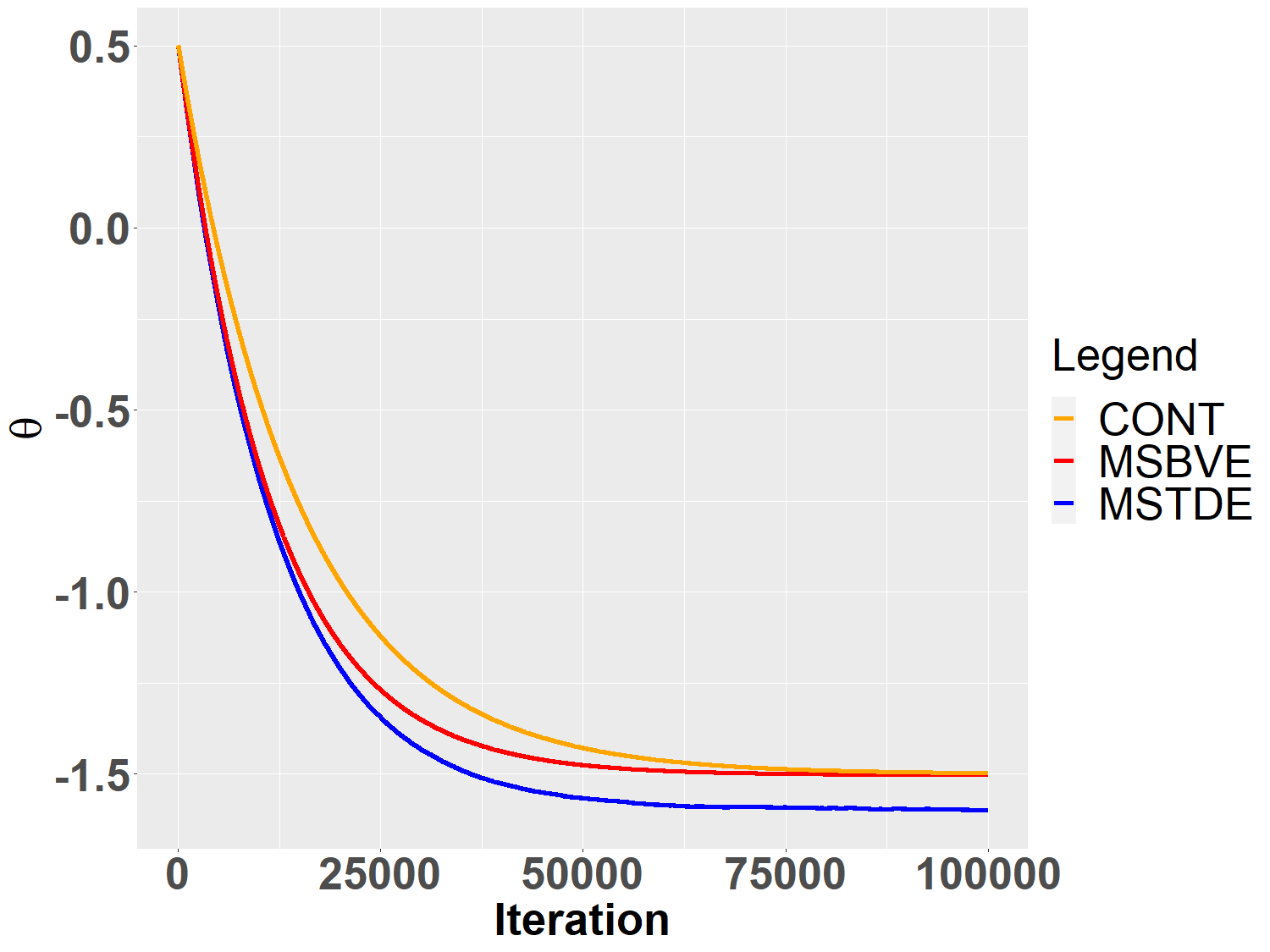}
        \caption{Linear example}
        \label{fig:1a}
    \end{subfigure}
    \hfill
    \begin{subfigure}[b]{0.32\textwidth}
        \centering
        \includegraphics[width=\textwidth]{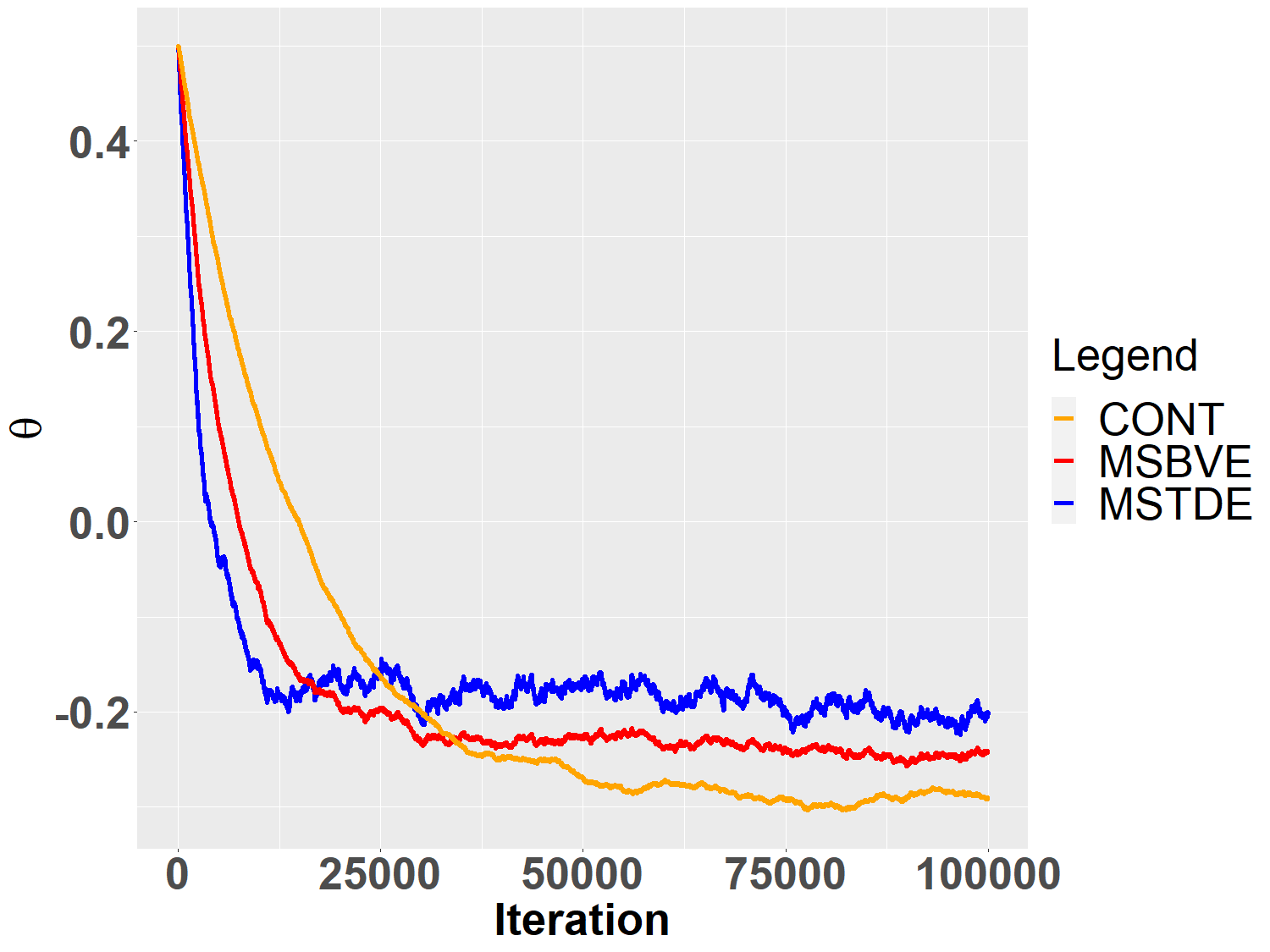}
        \caption{Quadratic example}
        \label{fig:1b}
    \end{subfigure}
        \begin{subfigure}[b]{0.32\textwidth}
        \centering
        \includegraphics[width=\textwidth]{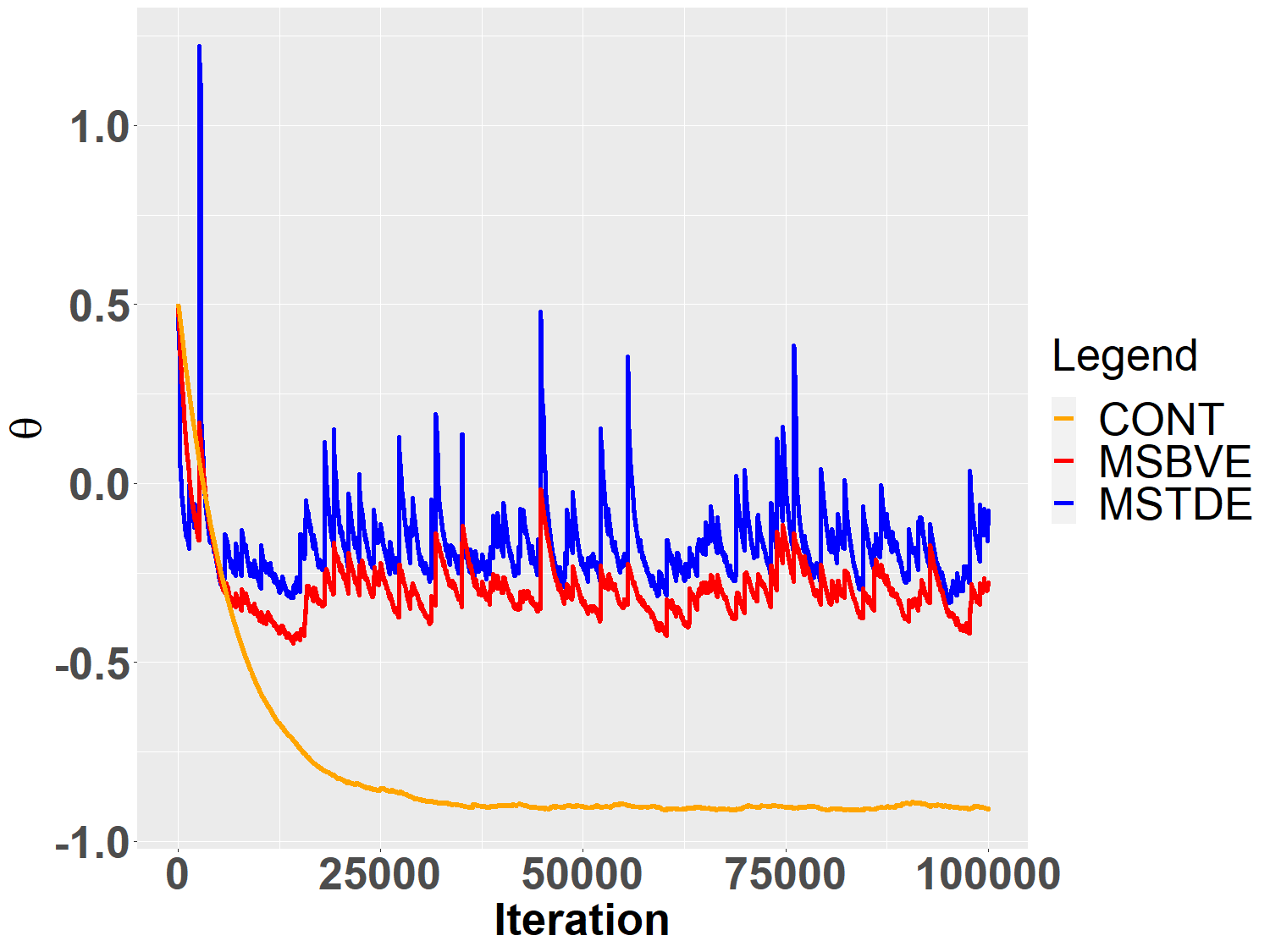}
        \caption{Exponential example}
        \label{fig:1c}
    \end{subfigure}
    \caption{ The convergence of $\theta$ of three different forms of value functions (linear, quadratic, and exponential)  }
    \label{fig:1}
\end{figure}

The results shown in Figure \ref{fig:1} illustrate the convergence behavior of $\theta$ under three different forms of the value function. In Figure \ref{fig:1}(a), where the value function is linear in $x$, the bias term $\text{bias}(\theta)$ vanishes. Consequently, the red line (MSBVE) and the orange line (Oracle estimator) converge to the same value, $-\frac{3}{2}$. However, as the value function has a high-degree polynomial, the bias increases. For instance, in Figure \ref{fig:1}(b), where the value function is quadratic in $x$, the red line (MSBVE) converges to $-\frac{40}{167}$, which is close to the blue line (MSTDE) that converges to $-\frac{8545}{45059}$. In contrast, the oracle estimator (orange line) converges to $-\frac{15}{52}$, showing a larger deviation. Similarly, in Figure \ref{fig:1}(c), where the value function is exponential in $x$, the red line (MSBVE) still converges closer to the oracle estimator than the blue line (MSTDE), although the gap between MSBVE and the oracle increases compared to the linear and quadratic cases.
 
These findings indicate that when the value function is a low-degree polynomial in $x$, the MSBVE algorithm demonstrates greater robustness to stochastic processes with jumps compared to MSTDE. Specifically, MSBVE converges to a quantity associated primarily with the continuous component of the process, even in the presence of jumps. When the value function is a high-degree polynomial in $x$, MSBVE still outperforms MSTDE; however, the performance gap narrows, and both estimators exhibit increased deviation from the oracle estimator.
Thus, it is important and interesting to develop a robust estimator that can handle the remaining bias from the gradient term of the value function, $\frac{\partial J^{\theta}(t,X_{t-})}{\partial x}$. 
We leave this for a future study.

\section{Data Application} \label{application}
In this section, we apply the proposed MSBVE algorithm and compare it with the existing MSTDE algorithm in the context of a mean-variance portfolio selection problem. We consider a financial market with two assets over the time interval $[0,T]$: the risk-free asset follows
\[dB_{t}=r_{f}B_{t}dt,\]
where $r_{f}$ is the risk-free rate,
and the risky asset follows
\[dS_{t}=S_{t}(\mu dt+\sigma dW_{t}),\]
where $\mu$ is the drift, $\sigma$ is the volatility, and $W_{t}$ is a standard Brownian motion. 
Let $X_{t}^{a}$ denote the wealth of an investor who allocates weight $a_{t}$ to the risky asset at time $t$. Following a simplified form of the classical Merton model for continuous-time portfolio selection, the wealth process then follows an SDE:
\[dX_{t}^{a}= a_{t}\left(\sigma\rho dt+\sigma dW_{t}\right),\]
where $\rho=\frac{\mu-r_{f}}{\sigma}$ is the Sharpe ratio.  See \citet{merton1971optimum} for the derivation of portfolio dynamics in a diffusion market. 
The investor aims to minimize the variance of the wealth at time $T$ while maintaining the expected value of the wealth at time $T$ at a certain level $z$. Specifically, 
\[\min_{a}\text{Var}(X^{a}_{T})\quad s.t.\  E(X_{T}^{a})=z.\]
The theoretical optimal policy from \cite{jia2022q} is given by 
\[a_{t}(\rho)=-\frac{\rho}{\sigma}(X_{t}-w(\rho)), \]
where $w(\rho)=\frac{ze^{\rho^{2}T}-X_{0}}{e^{\rho^{2}T}-1}$. Under this optimal policy, the wealth process becomes 
\[dX_{t}= -\rho(X_{t}-w)\left(\rho dt+dW_{t}\right),\]
and the corresponding value function is defined as: 
\begin{align*}
    J(t,x)&=E[X_{T}^{2}|X_{t}=x]-z^{2}\\
    &=(x-w)^2 e^{\rho^{2}(t-T)}-(w-z)^{2}.
\end{align*}
We parameterize the value function as
\[J^{\theta}(t,x)=(x-w)^2 e^{\theta^{2}(t-T)}-(w-z)^{2},\]
and define the corresponding parametric policy as follows:
\[a_{t}(\theta)=-\frac{\theta}{\hat{\sigma}}(x-w(\theta)),\]
where the estimated variance $\hat{\sigma}^{2}$ is given by 
\[\hat{\sigma}^{2}=\frac{\pi}{2}\sum_{i=1}^{n-1}|\Delta X_{t_{i}}|\cdot|\Delta X_{t_{i-1}}|.\]

Our theoretical analysis yields two key predictions: (1) in jump-free environments, MSBVE and MSTDE should exhibit comparable performance; (2) in settings with jumps, MSBVE is expected to outperform MSTDE due to its robustness to jumps.  
To empirically test both settings, we perform two types of experiments: one on raw S\&P 500 data (which includes jumps), and another on data where large jumps are removed. 

For jump detection, we apply the threshold $\tau=4\hat{\sigma}(\Delta t)^{0.47}$, and consider observations with $|\Delta X_{t}|>\tau$ to be large jumps  \citep{ait2020high}.
 In practice, however, this threshold is not applied directly to the wealth process $X_{t}$, but rather to the observed S\&P 500 price process $S_{t}$, since in real markets we only observe the stock price $S_{t}$ and the risk free rate $r_{f}$. Using this method, we identify days with large jumps in the dataset of 5-minute S\&P 500 data from 2010 to 2020. Notably, the period from 2012-10-01 to 2014-09-30 exhibits a higher frequency of jumps, with 195 jump days accounting for approximately 40\% of the sample. Thus, to check the effect of jumps, we focus on this two-year dataset for model training and testing. 
 
To simulate the jump-free scenario, we apply the thresholding procedure to remove large jumps. Specifically, we define the thresholded difference process as follows: 
\begin{align*}
    \Delta S_{t_{i}}^{T}=\begin{cases}
    \Delta S_{t_{i}}, &\text{if }|\Delta S_{t_{i}}|<\tau,\\
    0, &\text{if }|\Delta S_{t_{i}}|\geq\tau.
    \end{cases}
\end{align*}
We can obtain the thresholded process by summarizing the thresholded difference process $S_{t_{i}}^{T}=\sum_{j=0}^{i-1}\Delta S_{t_{j}}^{T}$. 
We then apply MSTDE and MSBVE to the thresholded data and compare the results with those obtained from the raw data.

We use 6 months (126 trading days) of 5-minute S\&P 500 data to train the allocation weights. The weights are updated after each trading day, and the algorithms are then tested on the following trading day. We consider a time horizon of $T=1$ day with time steps $\Delta t=1/79$. This process is repeated in a rolling-window fashion, advancing one day at a time. For each test day, we record the corresponding one-day terminal wealth, resulting in a sequence of terminal wealth values. Using this sequence, we compute the Sharpe ratio for both algorithms.
\begin{figure}[htbp]
    \centering
    \begin{subfigure}[b]{0.48\textwidth}
        \centering
        \includegraphics[width=\textwidth]{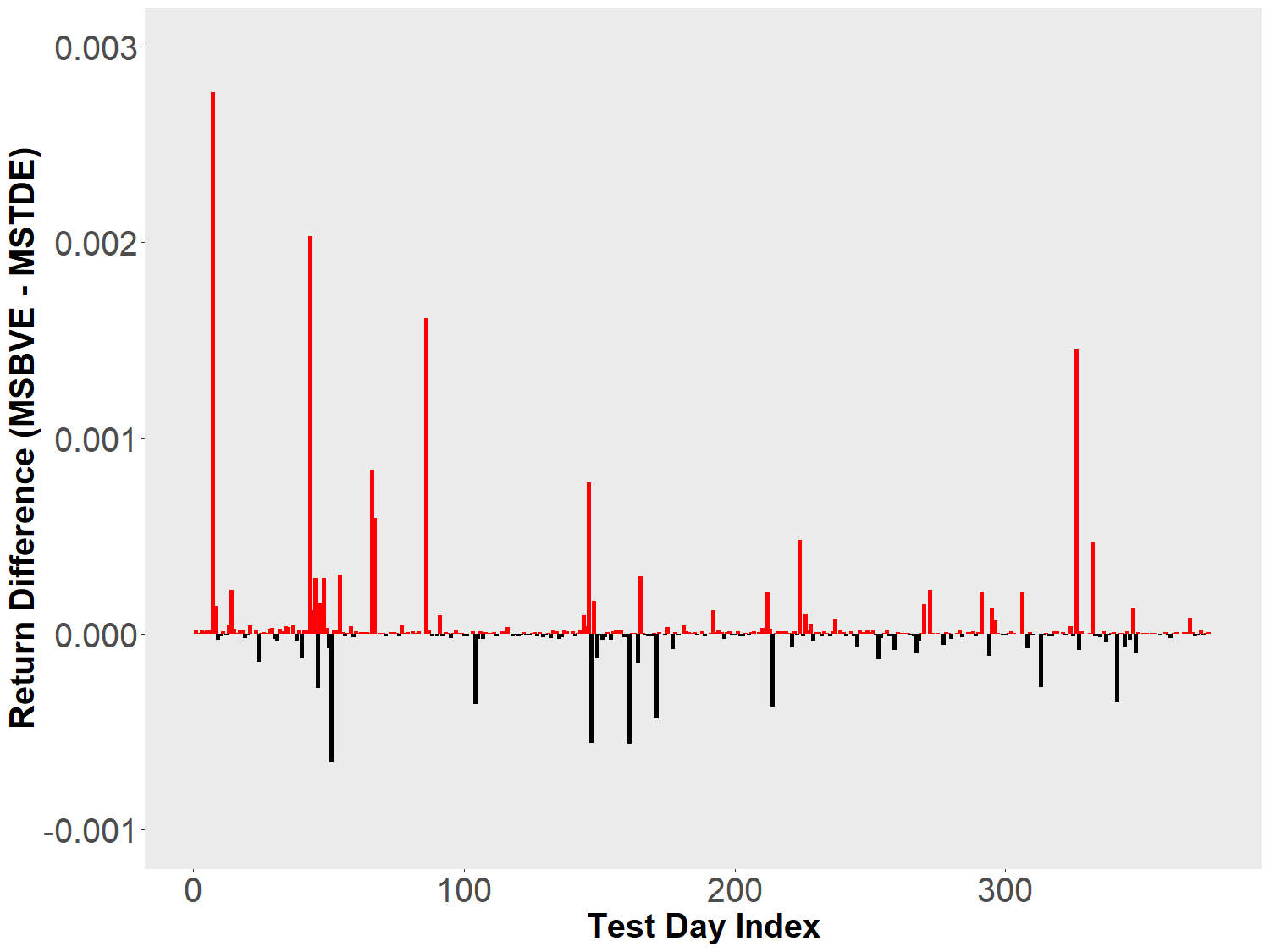}
        \caption{Raw S\&P 500 Data}
        \label{figure:2a}
    \end{subfigure}
    \hfill
    \begin{subfigure}[b]{0.48\textwidth}
        \centering
        \includegraphics[width=\textwidth]{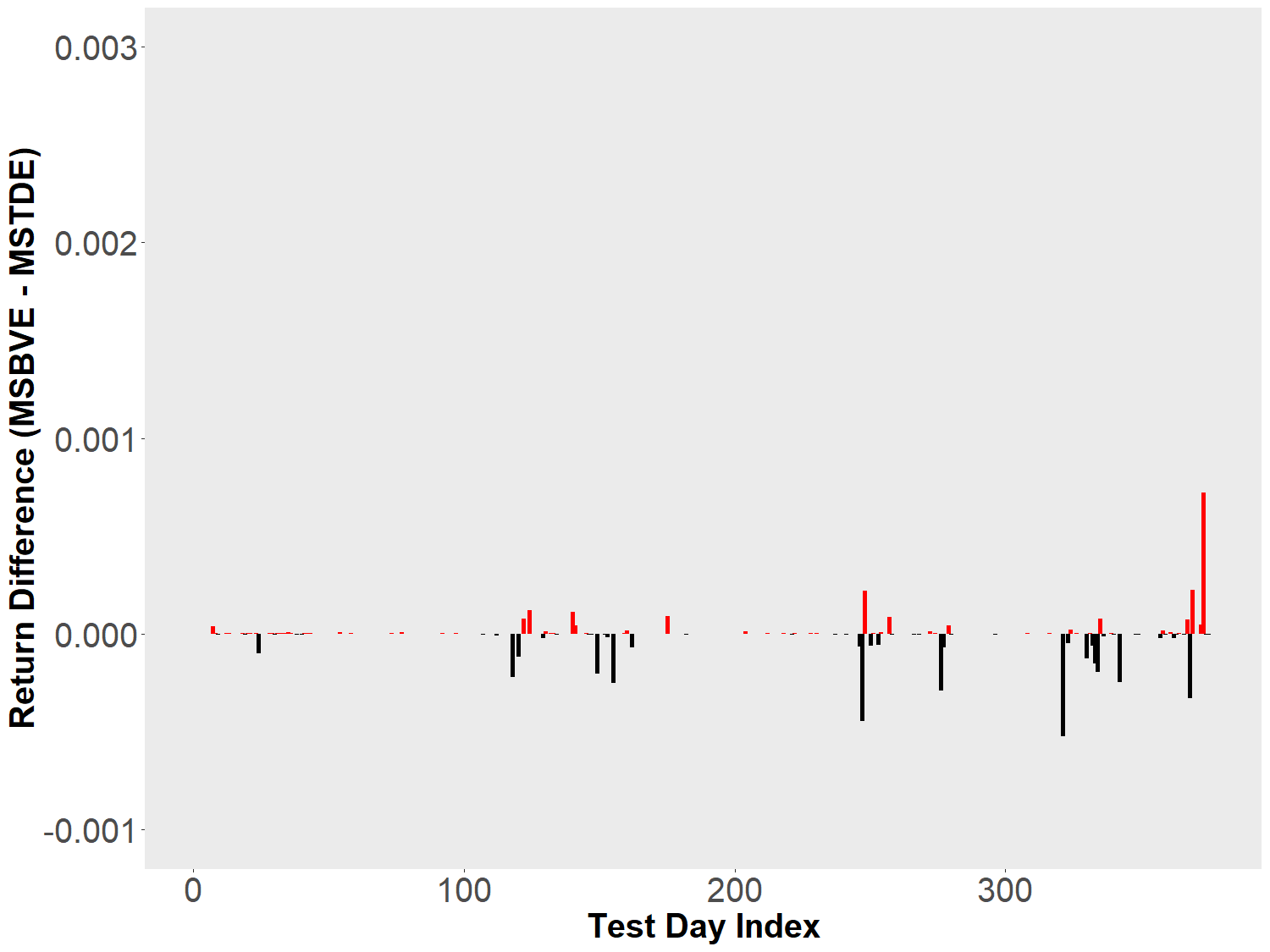}
        \caption{Thresholded Data}
        \label{figure:2b}
    \end{subfigure}
    \caption{Return differences between MSBVE and MSTDE. Note that red indicates instances where MSBVE outperforms MSTDE, while black indicates the opposite.}
    \label{figure:2}
\end{figure}

 Figure~\ref{figure:2} illustrates the difference in daily returns between MSBVE and MSTDE across the testing period. Each vertical line represents the return difference on a single test day, where red indicates instances where MSBVE outperforms MSTDE, while black indicates the opposite. In Figure~\ref{figure:2}(a), which corresponds to the raw data setting, most lines are red with substantial magnitudes of outperformance. This indicates that MSBVE consistently yields higher returns than MSTDE in the presence of jumps. This observation supports our theoretical findings that MSBVE is more robust in environments with jumps.
Figure~\ref{figure:2}(b) presents the performance differences after removing large jumps. In this case, red and black lines appear more balanced, and the return differences are notably smaller, which suggests that the two algorithms perform comparably in the absence of large jumps.

To quantitatively evaluate the performance of the algorithms, we compute the Sharpe ratio under both settings and report the results in Table~\ref{tab:comparison}. We also include q-learning algorithm for comparison, following the implementation in \cite{gao2024reinforcementlearningjumpdiffusionsfinancial}. In the raw data setting, MSBVE achieves a significantly higher Sharpe ratio than both MSTDE and q-learning, reaffirming its superior performance in environments with jumps. In contrast, under the jump-removed setting, the Sharpe ratios of the three algorithms are similar and lower than in the raw data case. These results further confirm that MSBVE performs best when applied directly to real financial data containing jumps, consistent with our theoretical findings.
\begin{table}[h]
\caption{Annualized Sharpe ratio comparison of MSBVE and MSTDE}
\label{tab:comparison}
\centering
\begin{tabular}{ccc}
\toprule
\textbf{Method} & \textbf{Raw S\&P 500 Data} & \textbf{Thresholded Data} \\
\midrule
MSTDE  & 1.041 & 0.587 \\
MSBVE  & 1.253 & 0.571 \\
q-learning & 1.055 & 0.599 \\
\bottomrule
\end{tabular}
\end{table}

\section{Conclusion} \label{sec-6}

In this paper, we explored the problem of estimating the value function in continuous-time reinforcement learning under stochastic dynamics with potential jump components. The continuous and jump parts exhibit fundamentally different characteristics, and in many applications, our primary interest lies in learning the dynamics of the continuous component. However, the latent continuous process is not directly observable due to the presence of jumps.
To address this challenge, we revisited the MSTDE algorithm and proposed a novel MSBVE algorithm, which updates a parameter vector by minimizing an alternative error metric.  
We have demonstrated that the MSBVE algorithm exhibits improved robustness and convergence behavior compared to the MSTDE algorithm. Importantly, when the underlying dynamics follow an SDE  without jumps, the two algorithms yield equivalent results. However, when the dynamics are governed by an SDE with jump components, MSBVE places greater emphasis on the continuous part of the dynamics and is less affected by jumps, whereas MSTDE directly reflects both the continuous and jump components, often resulting in larger bias. These findings highlight that adopting alternative error metrics, such as the mean bipower variation error, can enhance the performance and reliability of reinforcement learning algorithms in continuous-time settings, particularly in the presence of jump-diffusion processes.

 Moving forward, future research may explore extensions and refinements of the MSBVE algorithm, as well as investigate its applicability to real-world reinforcement learning problems characterized by complex dynamics and noisy environments. Additionally, further theoretical analysis and empirical studies could yield deeper insights into the convergence properties and computational efficiency of the proposed method, thereby contributing to the advancement of continuous-time reinforcement learning methodologies.

\addtolength{\textheight}{-.2in}%

\phantomsection\label{supplementary-material}
\bigskip

\begin{center}

{\large\bf SUPPLEMENTARY MATERIAL}

\end{center}

\begin{description}
\item[Appendix A -- Proofs:] 
Detailed proofs of all theoretical results stated in the main article. (PDF file)

\item[Appendix B -- Code and Data:] 
R code and datasets used to reproduce the results presented in the article. (Folder compressed as a ZIP file)

\end{description}

\bibliography{bibliography.bib}

@article{Barndorff-Nielsen2006,
  title={A Central Limit Theorem for Realised Power and Bipower Variations of Continuous Semimartingales},
  author={Barndorff-Nielsen, Ole E. and Shephard, Neil},
  journal={Journal of Financial Econometrics},
  volume={2},
  number={1},
  pages={1--37},
  year={2004},
  publisher={Oxford University Press}
}

@article{jiaandzhou2022,
  author       = {Yanwei Jia and
                  Xun Yu Zhou},
  title        = {Policy Evaluation and Temporal-Difference Learning in Continuous Time
                  and Space: {A} Martingale Approach},
  journal      = {CoRR},
  volume       = {abs/2108.06655},
  year         = {2021},
  eprinttype    = {arXiv},
  eprint       = {2108.06655},
}

@article{Doya:2000,
    author = {Kenji Doya},
    title = "{Reinforcement Learning in Continuous Time and Space}",
    journal = {Neural Computation},
    volume = {12},
    number = {1},
    pages = {219-245},
    year = {2000},
    month = {01},
    issn = {0899-7667},
}

@article{oh2024robust,
  title={Robust realized integrated beta estimator with application to dynamic analysis of integrated beta},
  author={Oh, Minseog and Kim, Donggyu and Wang, Yazhen},
  journal={Journal of Econometrics},
  pages={105810},
  year={2024},
  publisher={Elsevier}
}

@book{bellman1957dynamic,
  title={Dynamic Programming},
  author={Bellman, Richard},
  year={1957},
  publisher={Princeton University Press},
  address={Princeton, NJ},
}

@article{elkaroui1997,
  title={Backward stochastic differential equations in finance},
  author={El Karoui, N. and Mrad, M. and Mrad, M.},
  journal={Mathematical finance},
  volume={7},
  number={1},
  pages={1--71},
  year={1997},
}

@book{karatzas2014brownian,
  title={Brownian Motion and Stochastic Calculus},
  author={Karatzas, I. and Shreve, S.},
  isbn={9781461209492},
  series={Graduate Texts in Mathematics},
  year={2014},
  publisher={Springer New York}
}

@book{jacod_shiryaev_2003,
    author    = {Jean Jacod and Albert N. Shiryaev},
    title     = {Limit Theorems for Stochastic Processes},
    edition   = {2nd},
    year      = {2003},
    publisher = {Springer},
}

@book{shao2003mathematical,
  title={Mathematical Statistics},
  author={Shao, Jun},
  year={2003},
  edition={2nd},
  publisher={Springer}
}

@book{sutton2018reinforcement,
  title={Reinforcement Learning: An Introduction},
  author={Sutton, Richard S and Barto, Andrew G},
  year={2018},
  edition={2nd},
  publisher={MIT press},
  isbn={9780262039246}
}

@article{jia2022policy,
  title={Policy Gradient and Actor-Critic Learning in Continuous Time and Space: Theory and Algorithms},
  author={Jia, Huyuan and Zhou, Xun Yu},
  journal={Mathematical Control and Related Fields},
  volume={12},
  number={2},
  pages={305--338},
  year={2022},
  publisher={American Institute of Mathematical Sciences (AIMS)},
}

@article{jia2022q,
  title={Q-Learning in Continuous Time},
  author={Jia, Huyuan and Zhou, Xun Yu},
  journal={IEEE Transactions on Automatic Control},
  volume={67},
  number={6},
  pages={3102--3109},
  year={2022},
  publisher={IEEE},
}

@book{pham2009continuous,
  title={Continuous Time Stochastic Control and Optimization with Financial Applications},
  author={Pham, H.},
  year={2009},
  publisher={Springer},
  series={Stochastic Modelling and Applied Probability},
  volume={61},
  isbn={978-3-540-89335-5}
}

@article{Bo2024,
  author    = {Bo, L. and Huang, Y. and Yu, X. and Zhang, T.},
  title     = {Continuous-time Q-learning for jump-diffusion models under Tsallis entropy},
  journal   = {arXiv preprint},
  year      = {2024},
  eprint    = {2407.03888},
  archivePrefix = {arXiv},
  primaryClass = {cs.LG},
}

@article{Bender2023,
  author    = {Bender, C. and Thuan, N. T.},
  title     = {Entropy-regularized mean-variance portfolio optimization with jumps},
  journal   = {arXiv preprint},
  year      = {2023},
  eprint    = {2312.13409},
  archivePrefix = {arXiv},
  primaryClass = {q-fin.PM},
}

@article{guo2022,
author = {Guo, Xin and Hu, Anran and Zhang, Yufei},
title = {Reinforcement Learning for Linear-Convex Models with Jumps via Stability Analysis of Feedback Controls},
journal = {SIAM Journal on Control and Optimization},
volume = {61},
number = {2},
pages = {755-787},
year = {2023},
}

@article{Denkert2024,
  author    = {Denkert, R. and Pham, H. and Warin, X.},
  title     = {Control randomisation approach for policy gradient and application to reinforcement learning in optimal switching},
  journal   = {arXiv preprint},
  year      = {2024},
  eprint    = {2404.17939},
  archivePrefix = {arXiv},
  primaryClass = {cs.LG},
}

@article{gao2024reinforcementlearningjumpdiffusionsfinancial,
      title={Reinforcement Learning for Jump-Diffusions, with Financial Applications}, 
      author={Xuefeng Gao and Lingfei Li and Xun Yu Zhou},
      year={2024},
      eprint={2405.16449},
      journal   = {arXiv preprint},
      archivePrefix={arXiv},
      primaryClass={cs.LG},
}

@article{ait2020high,
  title={High-frequency factor models and regressions},
  author={A{\"i}t-Sahalia, Yacine and Kalnina, Ieva and Xiu, Dacheng},
  journal={Journal of Econometrics},
  volume={216},
  number={1},
  pages={86--105},
  year={2020},
  publisher={Elsevier}
}

@article{ait2017using,
  title={Using principal component analysis to estimate a high dimensional factor model with high-frequency data},
  author={A{\"i}t-Sahalia, Yacine and Xiu, Dacheng},
  journal={Journal of Econometrics},
  volume={201},
  number={2},
  pages={384--399},
  year={2017},
  publisher={Elsevier}
}

@article{andersen2007roughing,
  title={Roughing it up: Including jump components in the measurement, modeling, and forecasting of return volatility},
  author={Andersen, Torben G and Bollerslev, Tim and Diebold, Francis X},
  journal={The review of economics and statistics},
  volume={89},
  number={4},
  pages={701--720},
  year={2007},
  publisher={The MIT Press}
}

@article{corsi2010threshold,
  title={Threshold bipower variation and the impact of jumps on volatility forecasting},
  author={Corsi, Fulvio and Pirino, Davide and Reno, Roberto},
  journal={Journal of Econometrics},
  volume={159},
  number={2},
  pages={276--288},
  year={2010},
  publisher={Elsevier}
}

@article{merton1971optimum,
  title={Optimum consumption and portfolio rules in a continuous-time model},
  author={Merton, Robert C.},
  journal={Journal of Economic Theory},
  volume={3},
  number={4},
  pages={373--413},
  year={1971},
  publisher={Elsevier}
}
\nocite{Doya:2000}
\nocite{bellman1957dynamic}
\nocite{elkaroui1997}
\nocite{karatzas2014brownian}
\nocite{sutton2018reinforcement}
\nocite{jia2022policy}
\nocite{jia2022q}
\nocite{pham2009continuous}
\nocite{guo2022}
\nocite{Bender2023}
\nocite{Bo2024}
\nocite{Denkert2024}
\nocite{gao2024reinforcementlearningjumpdiffusionsfinancial}
\nocite{ait2017using}
\nocite{ait2020high}

\newpage
\appendix
\section{Proofs}

In this appendix, we prove  Theorem \ref{thm1} and Theorem \ref{thm2}. 
 We first recall a lemma from \cite{jiaandzhou2022}.

\begin{lemma}[Lemma 8 in \cite{jiaandzhou2022}]\label{lemma3}
 Let $f_{h}(x)=f(x)+r_{h}(x)$, where $f$ is a continuous function and $r_{h}$ converges to 0 uniformly on any compact set as $h\rightarrow 0$.
\begin{itemize}
    \item [(a)] Suppose $x_{h}^{*}\in\arg\min_{x}f_{h}(x)\neq\emptyset$ and $\lim_{h\rightarrow 0}x_{h}^{*}=x^{*}$. Then $x^{*}\in\arg\min_{x}f(x)$.
    \item [(b)] Suppose $f_{h}(x_{h}^{*})=0$ and $\lim_{h\rightarrow 0}x_{h}^{*}=x^{*}$. Then $f(x^{*})=0$.
\end{itemize}
\end{lemma}

\begin{proof}
    Refer to Appendix D in \cite{jiaandzhou2022}.
\end{proof}

 We present the following inequalities that will be useful for the proof of Theorem \ref{thm1} and Theorem \ref{thm2}.

\begin{lemma}[Burkholder-Davis-Gundy inequality]\label{lemma4}
For any $1\leq p<\infty$, there exists positive constants $c_{p},C_{p}$ such that, for all local martingales $M$ with $M_{0}=0$ and stopping time $\tau$, the following inequality holds:
\begin{align*}
    c_{p}E[\langle M\rangle_{\tau}^{p/2}]\leq E[|M_{\tau}^{\star}|^{p}]\leq C_{p}E[\langle M\rangle_{\tau}^{p/2}],
\end{align*}
where $\langle M\rangle_{t}$ denotes the quadratic variation of $M$, $M_{t}^{\star}=\sup\limits_{0\leq s\leq t}|M_{s}|$ denotes the maximum process. In particular, if we consider $\tau=T$, then
\begin{align}\label{bdg}
    E[|M_{T}|^{p}]\leq E[|M_{T}^{\star}|^{p}]\leq C_{p}E[\langle M\rangle_{T}^{p/2}].
\end{align}
\end{lemma}

\begin{lemma}[H\"{o}lder's inequality]\label{lemma5}
Let $p,q\in[1,+\infty]$ with $1/p+1/q=1$.
\begin{itemize}
    \item [(a)] Let $u,v\in\mathbb{R}^{n}$, then we have:
\begin{align}\label{holder1}
    \sum_{i=1}^{n}|u_{i}v_{i}|\le\left(\sum_{i=1}^{n}|u_{i}|^{p}\right)^{1/p}\left(\sum_{i=1}^{n}|v_{i}|^{q}\right)^{1/q}.
\end{align}
    \item [(b)] Let $Y,Z$ be random variables, then we have:
\begin{align}\label{holder2}
    E|YZ|\leq \left(E|Y|^{p}\right)^{1/p}\left(E|Z|^{q}\right)^{1/q}.
\end{align}
    \item [(c)] Let $f,g$ be real-valued integrable functions, then we have:
\begin{align}\label{holder3}
    \int |f(x)g(x)|dx\leq\left(\int |f(x)|^{p}dx\right)^{1/p}\left(\int |g(x)|^{q}dx\right)^{1/q}.
\end{align}
\end{itemize}
\end{lemma} 

\begin{lemma}\label{lemma6}
    Let $a,b,c,d\in\mathbb{R}$, if $bd=0$, then we have:
\begin{align}\label{abs}
    \Big||a+b||c+d|-|a||c|\Big|\leq|a||d|+|b||c|.
\end{align}
\end{lemma}

\begin{proof}
    If $b=0$, 
\begin{align*}
    & |c|-|d|\leq |c+d|\leq |c|+|d|\\
    \Longrightarrow\quad & |a||c|-|a||d|\leq |a||c+d|\leq |a||c|+|a||d|\\
    \Longrightarrow\quad & -|a||d|\leq |a||c+d|-|a||c|\leq |a||d|\\
    \Longrightarrow\quad & \Big||a||c+d|-|a||c|\Big|\leq |a||d|\\
    \Longrightarrow\quad & \Big||a+b||c+d|-|a||c|\Big|\leq |a||d|+|b||c|.
\end{align*}
If $d=0$,
\begin{align*}
    & |a|-|b|\leq |a+b|\leq |a|+|b|\\
    \Longrightarrow\quad & |a||c|-|b||c|\leq |c||a+b|\leq |a||c|+|b||c|\\
    \Longrightarrow\quad & -|b||c|\leq |c||a+b|-|a||c|\leq |b||c|\\
    \Longrightarrow\quad & \Big||c||a+b|-|a||c|\Big|\leq |b||c|\\
    \Longrightarrow\quad & \Big||a+b||c+d|-|a||c|\Big|\leq |a||d|+|b||c|.
\end{align*}
\end{proof}

We present the following standard results for asymptotic convergence and It\^o process.

\begin{lemma}[Theorem 1.8 in \cite{shao2003mathematical}]\label{lemma7} Let $X_{n}$ be a sequence of random variables, suppose that $X_{n}\stackrel{d}\rightarrow X$. Then for any $r>0$, 
\[\lim_{n\rightarrow\infty}E\Vert X_{n}\Vert_{r}^{r}=E\Vert X\Vert_{r}^{r}<\infty\]
if and only if $\{\Vert X_{n}\Vert_{r}^{r}\}$ is uniformly integrable in the sense that
\[\lim_{t\rightarrow\infty}\sup_{n}E(\Vert X_{n}\Vert_{r}^{r}I_{\{\Vert X_{n}\Vert_{r}>t\}})=0.\]
A sufficient condition for uniform integrability of $\{\Vert X_{n}\Vert_{r}^{r}\}$ is that
\[\sup_{n}E\Vert X_{n}\Vert_{r}^{r+\delta}<\infty\]
for a $\delta>0$.
\end{lemma}

\begin{lemma}\label{lemma8} Consider a semimartingale of the form $Y_{t}=Y_{0}+\int_{0}^{t}b_{s}ds+\int_{0}^{t}\sigma_{s}dW_{s}$ and the partition $0=t_{0}<t_{1}<\cdots<t_{n}=T$ with $t_{i}=T/n$, we have
\begin{itemize}
    \item [(a)] (Theorem 4.47 in \cite{jacod_shiryaev_2003}) 
\[\sum_{i=1}^{n}(Y_{t_{i}}-Y_{t_{i-1}})^{2}\stackrel{p}\rightarrow\int_{0}^{T}\sigma_{s}^{2}ds.\]
    \item [(b)] (Theorem 2.2 in \cite{Barndorff-Nielsen2006})
\[\sum_{i=1}^{n-1}|Y_{t_{i+1}}-Y_{t_{i}}||Y_{t_{i}}-Y_{t_{i-1}}|\stackrel{p}\rightarrow\frac{2}{\pi}\int_{0}^{T}\sigma_{s}^{2}ds.\]
\end{itemize}
\end{lemma}

Now, we start to prove Theorem \ref{thm1}.
\begin{proof}
Let us start by highlighting an useful fact. First fix a sufficiently small time grid size $\Delta t$, by Assumption \ref{ass1}(v), when $\Delta t$ is small enough, there is at most one jump within each interval $[t_{i},t_{i+1}]$, hence
\begin{align*}
    &\int_{t_{i}}^{t_{i+1}}\left[J^{\theta}(t,X_{t})-J^{\theta}(t,X_{t-})\right]dN_{t}\\
    &=\left\{\begin{array}{rcl}
    &0, &\text{if there is no jump in }[t_{i},t_{i+1}]\\
    &J^{\theta}\left(s_{i},X_{s_{i}}\right)-J^{\theta}\left(s_{i}-,X_{s_{i}-}\right), &\text{if there is a jump happens at }s_{i}\in[t_{i},t_{i+1}].
    \end{array}
    \right.
\end{align*}
This implies for any $K>0$,
\begin{align}\label{jumpproof}
    \left\{\int_{t_{i}}^{t_{i+1}}[J^{\theta}(t,X_{t})-J^{\theta}(t,X_{t-})]dN_{t}\right\}^{K}=\int_{t_{i}}^{t_{i+1}}\left[J^{\theta}(t,X_{t})-J^{\theta}(t,X_{t-})\right]^{K}dN_{t}.
\end{align}
Now let's prove the theorem. By It\^o's formula on $J^{\theta}(t_{i+1},X_{t_{i+1}})$, we have

\begin{align*}
    & \sum\limits_{i=0}^{n-1}\left[J^{\theta}(t_{i+1},X_{t_{i+1}})-J^{\theta}(t_{i},X_{t_{i}})\right]^{2} \\
    & = \sum\limits_{i=0}^{n-1}\left\{\int_{t_{i}}^{t_{i+1}}\mathcal{L}J^{\theta}(t,X_{t-})dt+\int_{t_{i}}^{t_{i+1}}\frac{\partial J^{\theta}}{\partial x}(t,X_{t-})^{\dagger}\sigma(t,X_{t-})dW_{t}\right.\\
    &\left.+\int_{t_{i}}^{t_{i+1}}\left[J^{\theta}(t,X_{t})-J^{\theta}(t,X_{t-})\right]dN_{t}\right\}^{2}\\
    & =\sum\limits_{i=0}^{n-1}\left\{\left(\int_{t_{i}}^{t_{i+1}}\mathcal{L}J^{\theta}(t,X_{t-})dt\right)^{2}+\left[\int_{t_{i}}^{t_{i+1}}\frac{\partial J^{\theta}}{\partial x}(t,X_{t-})^{\dagger}\sigma(t,X_{t-})dW_{t}\right]^{2}\right.\\
    & +\int_{t_{i}}^{t_{i+1}}\left[J^{\theta}(t,X_{t})-J^{\theta}(t,X_{t-})\right]^{2}dN_{t}\\
    &+2\left(\int_{t_{i}}^{t_{i+1}}\mathcal{L}J^{\theta}(t,X_{t-})dt\right)\left(\int_{t_{i}}^{t_{i+1}}\frac{\partial J^{\theta}}{\partial x}(t,X_{t-})^{\dagger}\sigma(t,X_{t-})dW_{t}\right)\\
    & +2\left\{\int_{t_{i}}^{t_{i+1}}\left[J^{\theta}(t,X_{t})-J^{\theta}(t,X_{t-})\right]dN_{t}\right\}\\
    &\quad\left.\cdot\left[\int_{t_{i}}^{t_{i+1}}\mathcal{L}J^{\theta}(t,X_{t-})dt+\int_{t_{i}}^{t_{i+1}}\frac{\partial J^{\theta}}{\partial x}(t,X_{t-})^{\dagger}\sigma(t,X_{t-})dW_{t})\right]\right\}.
\end{align*}
By It\^o's isometry,
\begin{align*}
    E\left\{\sum\limits_{i=0}^{n-1}\left[\int_{t_{i}}^{t_{i+1}}\frac{\partial J^{\theta}}{\partial x}(t,X_{t-})^{\dagger}\sigma(t,X_{t-})dW_{t}\right]^{2}\right\} & =E\left\{\sum\limits_{i=0}^{n-1}\int_{t_{i}}^{t_{i+1}}\left[\frac{\partial J^{\theta}}{\partial x}(t,X_{t-})^{\dagger}\sigma(t,X_{t-})\right]^{2}dt\right\}\\
    & = E\left\{\int_{0}^{T}\left[\frac{\partial J^{\theta}}{\partial x}(t,X_{t-})^{\dagger}\sigma(t,X_{t-})\right]^{2}dt\right\}.
\end{align*}
Thus,
\begin{align*}
    & \text{MSTDE}_{\Delta t}^{*}(\theta)\\
    & = E\left\{\int_{0}^{T}\left[\frac{\partial J^{\theta}}{\partial x}(t,X_{t-})^{\dagger}\sigma(t,X_{t-})\right]^{2}dt+\int_{0}^{T}\left[J^{\theta}(t,X_{t})-J^{\theta}(t,X_{t-})\right]^{2}dN_{t}\right\}\\
    & + \sum\limits_{i=0}^{n-1}E\left\{\left[\int_{t_{i}}^{t_{i+1}}\mathcal{L}J^{\theta}(t,X_{t-})dt\right]^{2}\right.\\
    &\left.+2\left[\int_{t_{i}}^{t_{i+1}}\mathcal{L}J^{\theta}(t,X_{t-})dt\right]\left[\int_{t_{i}}^{t_{i+1}}\frac{\partial J^{\theta}}{\partial x}(t,X_{t-})^{\dagger}\sigma(t,X_{t-})dW_{t}\right]\right.\\
    & +2\left[\int_{t_{i}}^{t_{i+1}}\left[J^{\theta}(t,X_{t})-J^{\theta}(t,X_{t-})\right]dN_{t}\right]\\
    &\quad\cdot\left.\left[\int_{t_{i}}^{t_{i+1}}\mathcal{L}J^{\theta}(t,X_{t-})dt+\int_{t_{i}}^{t_{i+1}}\frac{\partial J^{\theta}}{\partial x}(t,X_{t-})^{\dagger}\sigma(t,X_{t-})dW_{t}\right]\right\}.
\end{align*}
We write $\text{MSTDE}_{\Delta t}^{*}(\theta)=QV(\theta)+R_{1}(\theta)+R_{2}(\theta)$, where
\[QV(\theta):=E\left\{\int_{0}^{T}\left[\frac{\partial J^{\theta}}{\partial x}(t,X_{t-})^{\dagger}\sigma(t,X_{t-})\right]^{2}dt+\int_{0}^{T}\left[J^{\theta}(t,X_{t})-J^{\theta}(t,X_{t-})\right]^{2}dN_{t}\right\},\]
\begin{align*}
    R_{1}(\theta):&= \sum\limits_{i=0}^{n-1}E\left\{\left[\int_{t_{i}}^{t_{i+1}}\mathcal{L}J^{\theta}(t,X_{t-})dt\right]^{2}\right.\\
&+2\left.\left[\int_{t_{i}}^{t_{i+1}}\mathcal{L}J^{\theta}(t,X_{t-})dt\right]\left[\int_{t_{i}}^{t_{i+1}}\frac{\partial J^{\theta}}{\partial x}(t,X_{t-})^{\dagger}\sigma(t,X_{t-})dW_{t}\right]\right\},
\end{align*}
and
\begin{align*}
R_{2}(\theta):=\sum\limits_{i=0}^{n-1}E&\left\{  2\left[\int_{t_{i}}^{t_{i+1}}\left[J^{\theta}(t,X_{t})-J^{\theta}(t,X_{t-})\right]dN_{t}\right]\right.\\
&\left.\cdot\left[\int_{t_{i}}^{t_{i+1}}\mathcal{L}J^{\theta}(t,X_{t-})dt+\int_{t_{i}}^{t_{i+1}}\frac{\partial J^{\theta}}{\partial x}(t,X_{t-})^{\dagger}\sigma(t,X_{t-})dW_{t}\right]\right\}.
\end{align*}
By H\"{o}lder's inequality, we obtain:
\begin{align*}
    |R_{1}(\theta)|&\leq\sum_{i=0}^{n-1}E\left\{\int_{t_{i}}^{t_{i+1}}\left[\mathcal{L}J^{\theta}(t,X_{t-})\right]^{2}dt\Delta t\right\}\\
    & +2\sum_{i=0}^{n-1}\left\{E\left[\int_{t_{i}}^{t_{i+1}}\mathcal{L}J^{\theta}(t,X_{t-})dt\right]^{2}E\left[\int_{t_{i}}^{t_{i+1}}\frac{\partial J^{\theta}}{\partial x}(t,X_{t-})^{\dagger}\sigma(t,X_{t-})dW_{t}\right]^{2}\right\}^{1/2}\\
    &\leq\sum_{i=0}^{n-1}E\left\{\int_{t_{i}}^{t_{i+1}}\left[\mathcal{L}J^{\theta}(t,X_{t-})\right]^{2}dt\Delta t\right\}\\
    & + 2\sum_{i=0}^{n-1}\left\{E\left[\int_{t_{i}}^{t_{i+1}}\left(\mathcal{L}J^{\theta}(t,X_{t-})\right)^{2}dt\Delta t\right]E\left[\int_{t_{i}}^{t_{i+1}}\left(\frac{\partial J^{\theta}}{\partial x}(t,X_{t-})^{\dagger}\sigma(t,X_{t-})\right)^{2}dt\right]\right\}^{1/2}\\
    & \leq\sum_{i=0}^{n-1}E\left\{\int_{t_{i}}^{t_{i+1}}\left[\mathcal{L}J^{\theta}(t,X_{t-})\right]^{2}dt\Delta t\right\}\\
    & + 2\left\{\sum_{i=0}^{n-1}E\left[\int_{t_{i}}^{t_{i+1}}\left(\mathcal{L}J^{\theta}(t,X_{t-})\right)^{2}dt\Delta t\right]\right.\\
    &\quad\cdot\left.\sum_{i=0}^{n-1}E\left[\int_{t_{i}}^{t_{i+1}}\left(\frac{\partial J^{\theta}}{\partial x}(t,X_{t-})^{\dagger}\sigma(t,X_{t-})\right)^{2}dt\right]\right\}^{1/2}\\
    & = \Delta t\Vert\mathcal{L}J^{\theta}(\cdot,X_{\cdot})\Vert_{L^{2}}^{2}+2\sqrt{\Delta t}\Vert\mathcal{L}J^{\theta}(\cdot,X_{\cdot})\Vert_{L^{2}}\Vert\frac{\partial J^{\theta}}{\partial x}(\cdot,X_{\cdot})^{\dagger}\sigma(\cdot,X_{\cdot})\Vert_{L^{2}},
\end{align*}
where the first line is due to (\ref{holder3}) with $p=2,q=2,f=\mathcal{L}J^{\theta}(t,X_{t-})$ and $g=1$, the third line is due to (\ref{holder3}) with $p=2,q=2,f=\mathcal{L}J^{\theta}(t,X_{t-}),g=1$ and It\^o's isometry, and the fifth line is due to (\ref{holder1}) with $p=2,q=2,u_{i}=\left\{E\left[\int_{t_{i}}^{t_{i+1}}\left(\mathcal{L}J^{\theta}(t,X_{t-})\right)^{2}dt\Delta t\right]\right\}^{1/2}$ and $v_{i}=\left\{E\left[\int_{t_{i}}^{t_{i+1}}\left(\frac{\partial J^{\theta}}{\partial x}(t,X_{t-})^{\dagger}\sigma(t,X_{t-})\right)^{2}dt\right]\right\}^{1/2}$.

 Assumption \ref{ass3} states that the $L^{2}$-norm of $\mathcal{L}J^{\theta}$ and $|\frac{\partial J}{\partial x}\cdot\sigma|^{2}$ are continuous functions of $\theta$, thus, for any compact set $\Gamma$, their supremes are all finite. Hence, for all $\theta\in\Gamma$,
\begin{align*}
    |R_{1}(\theta)|&\leq\Delta t\left[\sup_{\theta\in\Gamma}\Vert\mathcal{L}J^{\theta}(\cdot,X_{\cdot})\Vert_{L^{2}}\right]^{2}\\
&+2\sqrt{\Delta t}\sup_{\theta\in\Gamma}\Vert\mathcal{L}J^{\theta}(\cdot,X_{\cdot})\Vert_{L^{2}}\sup_{\theta\in\Gamma}\Vert\frac{\partial J^{\theta}}{\partial x}(\cdot,X_{\cdot})^{\dagger}\sigma(\cdot,X_{\cdot})\Vert_{L^{2}}\rightarrow 0\text{ as }\Delta t\rightarrow 0,
\end{align*}
Recall that 
\begin{align*}
R_{2}(\theta):=\sum\limits_{i=0}^{n-1}E&\left\{  2\left[\int_{t_{i}}^{t_{i+1}}\left[J^{\theta}(t,X_{t})-J^{\theta}(t,X_{t-})\right]dN_{t}\right]\right.\\
&\left.\cdot\left[\int_{t_{i}}^{t_{i+1}}\mathcal{L}J^{\theta}(t,X_{t-})dt+\int_{t_{i}}^{t_{i+1}}\frac{\partial J^{\theta}}{\partial x}(t,X_{t-})^{\dagger}\sigma(t,X_{t-})dW_{t}\right]\right\},
\end{align*}
we only need to show $R_{2}(\theta)\rightarrow 0$. The desired results follows from the following lemma.

\begin{lemma}\label{lemma9}
\begin{align}
    \sum\limits_{i=0}^{n-1}E&\left\{  \left[\int_{t_{i}}^{t_{i+1}}\left[J^{\theta}(t,X_{t})-J^{\theta}(t,X_{t-})\right]dN_{t}\right]\right.\notag\\
&\left.\cdot\left[\int_{t_{i}}^{t_{i+1}}\mathcal{L}J^{\theta}(t,X_{t-})dt+\int_{t_{i}}^{t_{i+1}}\frac{\partial J^{\theta}}{\partial x}(t,X_{t-})^{\dagger}\sigma(t,X_{t-})dW_{t}\right]\right\}\notag\\
&\leq 2\left\{\sup_{0\leq i\leq n-1}P\left(N_{t_{i+1}}-N_{t_{i}}=1\right)\cdot\sup_{0\leq t\leq T} E\left[J^{\theta}(t,X_{t})\right]^{2}\cdot T\right\}^{1/2}\Vert\mathcal{L}J^{\theta}(\cdot,X_{\cdot})\Vert_{L^{2}}\notag\\
    &+2(C\Delta t)^{\frac{1}{4}}\left\{\sup_{0\leq t\leq T} E\left|J^{\theta}(t,X_{t})\right|^{2}\right\}^{\frac{1}{2}}\left[\left\Vert\left|\frac{\partial J^{\theta}}{\partial x}(\cdot,X_{\cdot})^{\dagger}\sigma(\cdot,X_{\cdot})\right|^{2}\right\Vert_{L^{2}}\right]^{\frac{1}{2}}\notag\\
    &\quad\cdot
    \left[\frac{T}{
    \Delta t}\sup_{0\leq i\leq n-1}P\left(N_{t_{i+1}}-N_{t_{i}}=1\right)\right]^{\frac{3}{4}}.\label{r2}
\end{align}
\end{lemma}

\begin{proof}
By applying H\"{o}lder's inequality in the first summand of the LHS of (\ref{r2}), we obtain:
\begin{align*}
    &\sum\limits_{i=0}^{n-1}E\left\{  \left[\int_{t_{i}}^{t_{i+1}}\left[J^{\theta}(t,X_{t})-J^{\theta}(t,X_{t-})\right]dN_{t}\right]\int_{t_{i}}^{t_{i+1}}\mathcal{L}J^{\theta}(t,X_{t-})dt\right\}\\
    &\leq\sum_{i=0}^{n-1}\left\{E\left[\int_{t_{i}}^{t_{i+1}}\left[J^{\theta}(t,X_{t})-J^{\theta}(t,X_{t-})\right]dN_{t}\right]^{2}E\left[\int_{t_{i}}^{t_{i+1}}\mathcal{L}J^{\theta}(t,X_{t-})dt\right]^{2}\right\}^{1/2}\\
    &=\sum_{i=0}^{n-1}\left\{P\left(N_{t_{i+1}}-N_{t_{i}}=1\right)E\left[J^{\theta}(s_{i},X_{s_{i}})-J^{\theta}(s_{i}-,X_{s_{i}-})\right]^{2}E\left[\int_{t_{i}}^{t_{i+1}}\mathcal{L}J^{\theta}(t,X_{t-})dt\right]^{2}\right\}^{1/2}\\
    &\leq \left\{\sup_{0\leq i\leq n-1}P\left(N_{t_{i+1}}-N_{t_{i}}=1\right)\cdot4\sup_{0\leq t\leq T} E\left[J^{\theta}(t,X_{t})\right]^{2}\right\}^{1/2}\sum_{i=0}^{n-1}\left\{E\left[\int_{t_{i}}^{t_{i+1}}\mathcal{L}J^{\theta}(t,X_{t-})dt\right]^{2}\right\}^{1/2}\\
    &\leq\left\{\sup_{0\leq i\leq n-1}P\left(N_{t_{i+1}}-N_{t_{i}}=1\right)\cdot4\sup_{0\leq t\leq T} E\left[J^{\theta}(t,X_{t})\right]^{2}\right\}^{1/2}\\
    &\quad\cdot\sum_{i=0}^{n-1}\left\{E\left[\int_{t_{i}}^{t_{i+1}}\left[\mathcal{L}J^{\theta}(t,X_{t-})\right]^{2}dt\Delta t\right]\right\}^{1/2}\\
    &\leq\left\{\sup_{0\leq i\leq n-1}P\left(N_{t_{i+1}}-N_{t_{i}}=1\right)\cdot4\sup_{0\leq t\leq T} E\left[J^{\theta}(t,X_{t})\right]^{2}\right\}^{1/2}\\
    &\quad\cdot\left\{\sum_{i=0}^{n-1}E\left[\int_{t_{i}}^{t_{i+1}}\left[\mathcal{L}J^{\theta}(t,X_{t-})\right]^2 dt\right]\sum_{i=0}^{n-1}\Delta t\right\}^{1/2}\\
    &= 2\left\{\sup_{0\leq i\leq n-1}P\left(N_{t_{i+1}}-N_{t_{i}}=1\right)\cdot\sup_{0\leq t\leq T} E\left[J^{\theta}(t,X_{t})\right]^{2}\cdot T\right\}^{1/2}\Vert\mathcal{L}J^{\theta}(\cdot,X_{\cdot})\Vert_{L^{2}},
\end{align*}
where the second line is due to (\ref{holder2}) with $p=2,q=2,Y_{i}=\int_{t_{i}}^{t_{i+1}}\left[J^{\theta}(t,X_{t})-J^{\theta}(t,X_{t-})\right]dN_{t}$ and $Z_{i}=\int_{t_{i}}^{t_{i+1}}\mathcal{L}J^{\theta}(t,X_{t-})dt$, the third line is due to (\ref{jumpproof}), the fifth line is due to (\ref{holder3}) with $p=2,q=2,f=\mathcal{L}J^{\theta}(t,X_{t-})$ and $g=1$, and the seventh line is due to (\ref{holder1}) with $p=2,q=2,u_{i}=\left\{E\left[\int_{t_{i}}^{t_{i+1}}\left(\mathcal{L}J^{\theta}(t,X_{t-})\right)^{2}dt\right]\right\}^{1/2}$ and $v_{i}=(\Delta t)^{1/2}$.
Then for the second summand of the LHS of (\ref{r2}),
\begin{align*}
    &\sum\limits_{i=0}^{n-1}E\left\{  \left[\int_{t_{i}}^{t_{i+1}}\left[J^{\theta}(t,X_{t})-J^{\theta}(t,X_{t-})\right]dN_{t}\right]\int_{t_{i}}^{t_{i+1}}\frac{\partial J^{\theta}}{\partial x}(t,X_{t-})^{\dagger}\sigma(t,X_{t-})dW_{t}\right\}\\
    &\leq\sum_{i=0}^{n-1}\left[E\left|\int_{t_{i}}^{t_{i+1}}\left[J^{\theta}(t,X_{t})-J^{\theta}(t,X_{t-})\right]dN_{t}\right|^{\frac{4}{3}}\right]^{\frac{3}{4}}\left[E\left|\int_{t_{i}}^{t_{i+1}}\frac{\partial J^{\theta}}{\partial x}(t,X_{t-})^{\dagger}\sigma(t,X_{t-})dW_{t}\right|^{4}\right]^{\frac{1}{4}}
    \\
    &=\sum_{i=0}^{n-1}\Bigg\{\left[P\left(N_{t_{i+1}}-N_{t_{i}}=1\right)E\left|J^{\theta}(s_{i},X_{s_{i}})-J^{\theta}(s_{i}-,X_{s_{i}-})\right|^{\frac{4}{3}}\right]^{\frac{3}{4}}\\
    &\quad\cdot\left[E\left|\int_{t_{i}}^{t_{i+1}}\frac{\partial J^{\theta}}{\partial x}(t,X_{t-})^{\dagger}\sigma(t,X_{t-})dW_{t}\right|^{4}\right]^{\frac{1}{4}}\Bigg\}\\
    &\leq \sum_{i=0}^{n-1}\Bigg\{\left[P\left(N_{t_{i+1}}-N_{t_{i}}=1\right)E\left|J^{\theta}(s_{i},X_{s_{i}})-J^{\theta}(s_{i}-,X_{s_{i}-})\right|^{\frac{4}{3}}\right]^{\frac{3}{4}}\\
    &\quad\cdot \left[CE\left(\int_{t_{i}}^{t_{i+1}}\left|\frac{\partial J^{\theta}}{\partial x}(t,X_{t-})^{\dagger}\sigma(t,X_{t-})\right|^{2}dt\right)^{2}\right]^{1/4}\Bigg\}\\
    &\leq C^{\frac{1}{4}}\left\{\sup_{0\leq i\leq n-1}P\left(N_{t_{i+1}}-N_{t_{i}}=1\right)\cdot\sup_{0\leq t\leq T} E\left|2J^{\theta}(t,X_{t})\right|^{\frac{4}{3}}\right\}^{\frac{3}{4}}\\
    &\quad\cdot\sum_{i=0}^{n-1}\left[E\left(\int_{t_{i}}^{t_{i+1}}\left|\frac{\partial J^{\theta}}{\partial x}(t,X_{t-})^{\dagger}\sigma(t,X_{t-})\right|^{2}dt\right)^{2}\right]^{1/4}\\
    &\leq C^{\frac{1}{4}}\left\{\sup_{0\leq i\leq n-1}P\left(N_{t_{i+1}}-N_{t_{i}}=1\right)\cdot\sup_{0\leq t\leq T} E\left|2J^{\theta}(t,X_{t})\right|^{\frac{4}{3}}\right\}^{\frac{3}{4}}\\
    &\quad\cdot\sum_{i=0}^{n-1}\left[E\left(\int_{t_{i}}^{t_{i+1}}\left|\frac{\partial J^{\theta}}{\partial x}(t,X_{t-})^{\dagger}\sigma(t,X_{t-})\right|^{4}dt\right)\Delta t\right]^{1/4}\\
    &\leq (C\Delta t)^{\frac{1}{4}}\left\{\sup_{0\leq i\leq n-1}P\left(N_{t_{i+1}}-N_{t_{i}}=1\right)\cdot\sup_{0\leq t\leq T} E\left|2J^{\theta}(t,X_{t})\right|^{\frac{4}{3}}\right\}^{\frac{3}{4}}\\
    &\quad\cdot\left[\left(\sum_{i=0}^{n-1}E\int_{t_{i}}^{t_{i+1}}\left|\frac{\partial J^{\theta}}{\partial x}(t,X_{t-})^{\dagger}\sigma(t,X_{t-})\right|^{4}dt\right)^{1/4}\left(\sum_{i=0}^{n-1}1\right)^{\frac{3}{4}}\right]\\
    &\leq2(C\Delta t)^{\frac{1}{4}}\left\{\sup_{0\leq t\leq T} E\left|J^{\theta}(t,X_{t})\right|^{2}\right\}^{\frac{1}{2}}\left[\left\Vert\left|\frac{\partial J^{\theta}}{\partial x}(\cdot,X_{\cdot})^{\dagger}\sigma(\cdot,X_{\cdot})\right|^{2}\right\Vert_{L^{2}}\right]^{\frac{1}{2}}\\
    &\quad\cdot
    \left[\frac{T}{
    \Delta t}\sup_{0\leq i\leq n-1}P\left(N_{t_{i+1}}-N_{t_{i}}=1\right)\right]^{\frac{3}{4}},
\end{align*}
where the second line is due to (\ref{holder2}) with $p=4/3,q=4,Y_{i}=\int_{t_{i}}^{t_{i+1}}\left[J^{\theta}(t,X_{t})-J^{\theta}(t,X_{t-})\right]dN_{t}$ and $Z_{i}=\int_{t_{i}}^{t_{i+1}}\frac{\partial J^{\theta}}{\partial x}(t,X_{t-})^{\dagger}\sigma(t,X_{t-})dW_{t}$, the third line is due to (\ref{jumpproof}) with $K=4/3$, the fifth line is due to (\ref{bdg}) with $p=4,(M_{i})_{t}=\int_{t_{i}}^{t_{i}+t}\frac{\partial J^{\theta}}{\partial x}(t,X_{t-})^{\dagger}\sigma(t,X_{t-})dW_{t}$ and $C$ is a positive constant, the ninth line is due to (\ref{holder3}) with $p=2,q=2,f=\left|\frac{\partial J^{\theta}}{\partial x}(t,X_{t-})^{\dagger}\sigma(t,X_{t-})\right|^{2}$ and $g=1$, the eleventh line is due to (\ref{holder1}) with $p=4,q=4/3,u_{i}=\left(E\int_{t_{i}}^{t_{i+1}}\left|\frac{\partial J^{\theta}}{\partial x}(t,X_{t-})^{\dagger}\sigma(t,X_{t-})\right|^{4}\right)^{1/4}$ and $v_{i}=1$, and the thirteenth line is due to (\ref{holder2}) with $p=3/2,q=3,Y=|J^{\theta}(t,X_{t})|^{\frac{4}{3}},Z=1$ and the fact that $T=n\Delta t$.
Hence, we have:
\begin{align*}
    \sum\limits_{i=0}^{n-1}E&\left\{  \left[\int_{t_{i}}^{t_{i+1}}\left[J^{\theta}(t,X_{t})-J^{\theta}(t,X_{t-})\right]dN_{t}\right]\right.\\
&\left.\cdot\left[\int_{t_{i}}^{t_{i+1}}\mathcal{L}J^{\theta}(t,X_{t-})dt+\int_{t_{i}}^{t_{i+1}}\frac{\partial J^{\theta}}{\partial x}(t,X_{t-})^{\dagger}\sigma(t,X_{t-})dW_{t}\right]\right\}\\
&\leq 2\left\{\sup_{0\leq i\leq n-1}P\left(N_{t_{i+1}}-N_{t_{i}}=1\right)\cdot\sup_{0\leq t\leq T} E\left[J^{\theta}(t,X_{t})\right]^{2}\cdot T\right\}^{1/2}\Vert\mathcal{L}J^{\theta}(\cdot,X_{\cdot})\Vert_{L^{2}}\\
    &+2(C\Delta t)^{\frac{1}{4}}\left\{\sup_{0\leq t\leq T} E\left|J^{\theta}(t,X_{t})\right|^{2}\right\}^{\frac{1}{2}}\left[\left\Vert\left|\frac{\partial J^{\theta}}{\partial x}(\cdot,X_{\cdot})^{\dagger}\sigma(\cdot,X_{\cdot})\right|^{2}\right\Vert_{L^{2}}\right]^{\frac{1}{2}}\\
    &\quad\cdot
    \left[\frac{T}{
    \Delta t}\sup_{0\leq i\leq n-1}P\left(N_{t_{i+1}}-N_{t_{i}}=1\right)\right]^{\frac{3}{4}}.
\end{align*}
This completes the proof of Lemma \ref{lemma9}.
\end{proof}

Assumption \ref{ass3} implies that for any compact set $\Gamma$, $\sup_{0\leq t\leq T} \left\{E\left|J^{\theta}(t,X_{t})\right|^{2}\right\}^{\frac{1}{2}}$ is bounded for all $\theta\in\Gamma$. Assumption \ref{ass1}(v) states that $\sup\limits_{0\leq i\leq n-1}P\left(N_{t_{i+1}}-N_{t_{i}}=1\right)=\mathcal{O}(\Delta t)$. Hence, for all $\theta\in\Gamma$, by Lemma \ref{lemma9}, we have the following:
\begin{align*}
    &|R_{2}(\theta)|=2\sum\limits_{i=0}^{n-1}E\left\{  \left[\int_{t_{i}}^{t_{i+1}}\left[J^{\theta}(t,X_{t})-J^{\theta}(t,X_{t-})\right]dN_{t}\right]\right.\\
    &\left.\cdot\left[\int_{t_{i}}^{t_{i+1}}\mathcal{L}J^{\theta}(t,X_{t-})dt+\int_{t_{i}}^{t_{i+1}}\frac{\partial J^{\theta}}{\partial x}(t,X_{t-})^{\dagger}\sigma(t,X_{t-})dW_{t}\right]\right\}\\
    &\leq 4\left\{\sup_{0\leq i\leq n-1}P\left(N_{t_{i+1}}-N_{t_{i}}=1\right)\cdot\sup_{0\leq t\leq T} E\left[J^{\theta}(t,X_{t})\right]^{2}\cdot T\right\}^{1/2}\Vert\mathcal{L}J^{\theta}(\cdot,X_{\cdot})\Vert_{L^{2}}\\
    &+4(C\Delta t)^{\frac{1}{4}}\left\{\sup_{0\leq t\leq T} E\left|J^{\theta}(t,X_{t})\right|^{2}\right\}^{\frac{1}{2}}\left[\left\Vert\left|\frac{\partial J^{\theta}}{\partial x}(\cdot,X_{\cdot})^{\dagger}\sigma(\cdot,X_{\cdot})\right|^{2}\right\Vert_{L^{2}}\right]^{\frac{1}{2}}\\
    &\quad\cdot
    \left[\frac{T}{
    \Delta t}\sup_{0\leq i\leq n-1}P\left(N_{t_{i+1}}-N_{t_{i}}=1\right)\right]^{\frac{3}{4}}\rightarrow 0\quad\text{ as }\Delta t\rightarrow 0.
\end{align*}
The desired result follows from Lemma \ref{lemma3}.
\end{proof}

Now, we present the proof of Theorem \ref{thm2}.
\begin{proof}
Let us start by highlighting an useful fact. First fix a sufficiently small time grid size $\Delta t$, by Assumption \ref{ass1}(v), when $\Delta t$ is small enough, there is at most one jump within each set $[t_{i-1},t_{i}]\cup[t_{i},t_{i+1}]$, hence at least one of $\int_{t_{i}}^{t_{i+1}}|J^{\theta}\left(t,X_{t}\right)-J^{\theta}\left(t,X_{t-}\right)|dN_{t}$ and $\int_{t_{i-1}}^{t_{i}}|J^{\theta}\left(t,X_{t}\right)-J^{\theta}\left(t,X_{t-}\right)|dN_{t}$ is zero. Thus,
\begin{align}\label{jumpproof2}
    \left[\int_{t_{i}}^{t_{i+1}}\left|J^{\theta}(t,X_{t})-J^{\theta}(t,X_{t-})\right|dN_{t}\right]\left[\int_{t_{i-1}}^{t_{i}}\left|J^{\theta}(t,X_{t})-J^{\theta}(t,X_{t-})\right|dN_{t}\right]=0.
\end{align}
Now let's prove the theorem. By It\^o's formula on $J^{\theta}(t_{i+1},X_{t_{i+1}})$, we have
\begin{align*}
    & \sum\limits_{i=1}^{n-1}\left|J^{\theta}(t_{i+1},X_{t_{i+1}})-J^{\theta}(t_{i},X_{t_{i}})\right|\left|J^{\theta}(t_{i},X_{t_{i}})-J^{\theta}(t_{i-1},X_{t_{i-1}})\right|\\
    & = \sum\limits_{i=1}^{n-1}\left|\int_{t_{i}}^{t_{i+1}}\mathcal{L}J^{\theta}(t,X_{t-})dt+\int_{t_{i}}^{t_{i+1}}\frac{\partial J^{\theta}}{\partial x}(t,X_{t-})^{\dagger}\sigma(t,X_{t-})dW_{t}\right.\\
    &+\left.\int_{t_{i}}^{t_{i+1}}\left[J^{\theta}(t,X_{t})-J^{\theta}(t,X_{t-})\right]dN_{t}\right|\\
    &\cdot\left|\int_{t_{i-1}}^{t_{i}}\mathcal{L}J^{\theta}(t,X_{t-})dt+\int_{t_{i-1}}^{t_{i}}\frac{\partial J^{\theta}}{\partial x}(t,X_{t-})^{\dagger}\sigma(t,X_{t-})dW_{t}+\int_{t_{i-1}}^{t_{i}}\left[J^{\theta}(t,X_{t})-J^{\theta}(t,X_{t-})\right]dN_{t}\right|.
\end{align*}
We write $\text{MSBVE}_{\Delta t}(\theta)=BV(\theta)+\tilde{R}_{1}(\theta)+\tilde{R}_{2}(\theta)$, where
\[BV(\theta):=\frac{2}{\pi}E\left\{\int_{0}^{T}\left[\frac{\partial J^{\theta}}{\partial x}(t,X_{t-})^{\dagger}\sigma(t,X_{t-})\right]^{2}dt\right\},\]
\begin{align*}
    \tilde{R}_{1}(\theta):= & E\left[\sum\limits_{i=1}^{n-1}\left|\int_{t_{i}}^{t_{i+1}}\mathcal{L}J^{\theta}(t,X_{t-})dt+\int_{t_{i}}^{t_{i+1}}\frac{\partial J^{\theta}}{\partial x}(t,X_{t-})^{\dagger}\sigma(t,X_{t-})dW_{t}\right|\right.\\
    & \left.\cdot\left|\int_{t_{i-1}}^{t_{i}}\mathcal{L}J^{\theta}(t,X_{t-})dt+\int_{t_{i-1}}^{t_{i}}\frac{\partial J^{\theta}}{\partial x}(t,X_{t-})^{\dagger}\sigma(t,X_{t-})dW_{t}\right|\right]-BV(\theta),
\end{align*}
and
\begin{align*}
    \tilde{R}_{2}(\theta):= & E\left[\sum\limits_{i=1}^{n-1}\left|\int_{t_{i}}^{t_{i+1}}\mathcal{L}J^{\theta}(t,X_{t-})dt+\int_{t_{i}}^{t_{i+1}}\frac{\partial J^{\theta}}{\partial x}(t,X_{t-})^{\dagger}\sigma(t,X_{t-})dW_{t}\right.\right.\\
    &+\left.\int_{t_{i}}^{t_{i+1}}\left[J^{\theta}(t,X_{t})-J^{\theta}(t,X_{t-})\right]dN_{t}\right|\\
    &\cdot\left|\int_{t_{i-1}}^{t_{i}}\mathcal{L}J^{\theta}(t,X_{t-})dt+\int_{t_{i-1}}^{t_{i}}\frac{\partial J^{\theta}}{\partial x}(t,X_{t-})^{\dagger}\sigma(t,X_{t-})dW_{t}\right.\\
    &+\left.\left.\int_{t_{i-1}}^{t_{i}}\left[J^{\theta}(t,X_{t})-J^{\theta}(t,X_{t-})\right]dN_{t}\right|\right]\\
    &-E\left[\sum\limits_{i=1}^{n-1}\left|\int_{t_{i}}^{t_{i+1}}\mathcal{L}J^{\theta}(t,X_{t-})dt+\int_{t_{i}}^{t_{i+1}}\frac{\partial J^{\theta}}{\partial x}(t,X_{t-})^{\dagger}\sigma(t,X_{t-})dW_{t}\right|\right.\\
    & \left.\cdot\left|\int_{t_{i-1}}^{t_{i}}\mathcal{L}J^{\theta}(t,X_{t-})dt+\int_{t_{i-1}}^{t_{i}}\frac{\partial J^{\theta}}{\partial x}(t,X_{t-})^{\dagger}\sigma(t,X_{t-})dW_{t}\right|\right]
\end{align*}
By Lemma \ref{lemma8}(b), we have:
\begin{align*}
    & Y_{n}:=\sum\limits_{i=1}^{n-1}|\int_{t_{i}}^{t_{i+1}}\mathcal{L}J^{\theta}(t,X_{t-})dt+\int_{t_{i}}^{t_{i+1}}\frac{\partial J^{\theta}}{\partial x}(t,X_{t-})^{\dagger}\sigma(t,X_{t-})dW_{t}|\\
    & \cdot|\int_{t_{i-1}}^{t_{i}}\mathcal{L}J^{\theta}(t,X_{t-})dt+\int_{t_{i-1}}^{t_{i}}\frac{\partial J^{\theta}}{\partial x}(t,X_{t-})^{\dagger}\sigma(t,X_{t-})dW_{t}|\\
    &-\frac{2}{\pi}\int_{0}^{T}\left[\frac{\partial J^{\theta}}{\partial x}(t,X_{t-})^{\dagger}\sigma(t,X_{t-})\right]^{2}dt\stackrel{p}\rightarrow 0.
\end{align*}
This implies $Y_{n}\stackrel{d}\rightarrow 0$. Notice that
\begin{align*}
    &E|Y_{n}|^{2}\\
    &\leq 2E\left|\sum\limits_{i=1}^{n-1}|\int_{t_{i}}^{t_{i+1}}\mathcal{L}J^{\theta}(t,X_{t-})dt+\int_{t_{i}}^{t_{i+1}}\frac{\partial J^{\theta}}{\partial x}(t,X_{t-})^{\dagger}\sigma(t,X_{t-})dW_{t}|\right.\\
    & \left.\cdot|\int_{t_{i-1}}^{t_{i}}\mathcal{L}J^{\theta}(t,X_{t-})dt+\int_{t_{i-1}}^{t_{i}}\frac{\partial J^{\theta}}{\partial x}(t,X_{t-})^{\dagger}\sigma(t,X_{t-})dW_{t}|\right|^{2}\\
    &+2E\left[\frac{2}{\pi}\int_{0}^{T}\left(\frac{\partial J^{\theta}}{\partial x}(t,X_{t-})^{\dagger}\sigma(t,X_{t-})\right)^{2}dt\right]^{2}\\
    &\leq 2E\left[\sum\limits_{i=1}^{n-1}\left|\int_{t_{i}}^{t_{i+1}}\mathcal{L}J^{\theta}(t,X_{t-})dt+\int_{t_{i}}^{t_{i+1}}\frac{\partial J^{\theta}}{\partial x}(t,X_{t-})^{\dagger}\sigma(t,X_{t-})dW_{t}\right|^{2}\right.\\
    &\cdot\left.\sum\limits_{i=1}^{n-1}\left|\int_{t_{i-1}}^{t_{i}}\mathcal{L}J^{\theta}(t,X_{t-})dt+\int_{t_{i-1}}^{t_{i}}\frac{\partial J^{\theta}}{\partial x}(t,X_{t-})^{\dagger}\sigma(t,X_{t-})dW_{t}\right|^{2}\right]\\
    &+\frac{8}{\pi^{2}}E\left[\int_{0}^{T}\left(\frac{\partial J^{\theta}}{\partial x}(t,X_{t-})^{\dagger}\sigma(t,X_{t-})\right)^{4}dt\cdot T\right]\\
    &\leq 2E\left[\sum\limits_{i=0}^{n-1}\left|\int_{t_{i}}^{t_{i+1}}\mathcal{L}J^{\theta}(t,X_{t-})dt+\int_{t_{i}}^{t_{i+1}}\frac{\partial J^{\theta}}{\partial x}(t,X_{t-})^{\dagger}\sigma(t,X_{t-})dW_{t}\right|^{2}\right]^{2}\\
    &+\frac{8}{\pi^{2}}E\left[\int_{0}^{T}\left(\frac{\partial J^{\theta}}{\partial x}(t,X_{t-})^{\dagger}\sigma(t,X_{t-})\right)^{4}dt\cdot T\right]\\
    &\leq 2n\sum_{i=0}^{n-1}E\left|\int_{t_{i}}^{t_{i+1}}\mathcal{L}J^{\theta}(t,X_{t-})dt+\int_{t_{i}}^{t_{i+1}}\frac{\partial J^{\theta}}{\partial x}(t,X_{t-})^{\dagger}\sigma(t,X_{t-})dW_{t}\right|^{4}\\
    &+\frac{8}{\pi^{2}}E\left[\int_{0}^{T}\left(\frac{\partial J^{\theta}}{\partial x}(t,X_{t-})^{\dagger}\sigma(t,X_{t-})\right)^{4}dt\cdot T\right]\\
    &\leq 16n\sum_{i=0}^{n-1}\left(E\left|\int_{t_{i}}^{t_{i+1}}\mathcal{L}J^{\theta}(t,X_{t-})dt\right|^{4}+E\left|\int_{t_{i}}^{t_{i+1}}\frac{\partial J^{\theta}}{\partial x}(t,X_{t-})^{\dagger}\sigma(t,X_{t-})dW_{t}\right|^{4}\right)\\
    &+\frac{8}{\pi^{2}}E\left[\int_{0}^{T}\left(\frac{\partial J^{\theta}}{\partial x}(t,X_{t-})^{\dagger}\sigma(t,X_{t-})\right)^{4}dt\cdot T\right]\\
    &\leq 16n\sum_{i=0}^{n-1}\left[E\left(\int_{t_{i}}^{t_{i+1}}\left(\mathcal{L}J^{\theta}(t,X_{t-})\right)^{4}dt\cdot(\Delta t)^{3}\right)\right.\\
    &\left.+CE\left(\int_{t_{i}}^{t_{i+1}}\left|\frac{\partial J^{\theta}}{\partial x}(t,X_{t-})^{\dagger}\sigma(t,X_{t-})\right|^{2}dt\right)^{2}\right]\\
    &+\frac{8}{\pi^{2}}E\left[\int_{0}^{T}\left(\frac{\partial J^{\theta}}{\partial x}(t,X_{t-})^{\dagger}\sigma(t,X_{t-})\right)^{4}dt\cdot T\right]\\
    &\leq 16n\sum_{i=0}^{n-1}\left[(\Delta t)^{3}E\int_{t_{i}}^{t_{i+1}}\left(\mathcal{L}J^{\theta}(t,X_{t-})\right)^{4}dt+C(\Delta t) E\int_{t_{i}}^{t_{i+1}}\left|\frac{\partial J^{\theta}}{\partial x}(t,X_{t-})^{\dagger}\sigma(t,X_{t-})\right|^{4}dt\right]\\
    &+\frac{8}{\pi^{2}}E\left[\int_{0}^{T}\left(\frac{\partial J^{\theta}}{\partial x}(t,X_{t-})^{\dagger}\sigma(t,X_{t-})\right)^{4}dt\cdot T\right]\\
    &=16n(\Delta t)\left[(\Delta t)^{2}E\int_{0}^{T}(\mathcal{L}J^{\theta}(t,X_{t-}))^{4}dt+CE\int_{0}^{T}\left|\frac{\partial J^{\theta}}{\partial x}(t,X_{t-})^{\dagger}\sigma(t,X_{t-})\right|^{4}dt\right]\\
    &+\frac{8}{\pi^{2}}E\left[\int_{0}^{T}\left(\frac{\partial J^{\theta}}{\partial x}(t,X_{t-})^{\dagger}\sigma(t,X_{t-})\right)^{4}dt\cdot T\right]\\
    &=T\left[16(\Delta t)^{2}\Vert\mathcal{L}J^{\theta}(\cdot,X_{\cdot})^{2}\Vert_{L_{2}}^{2}+(16C+\frac{8}{\pi^{2}})\left\Vert\left|\frac{\partial J^{\theta}}{\partial x}(\cdot,X_{\cdot})^{\dagger}\sigma(\cdot,X_{\cdot})\right|^{2}\right\Vert_{L^{2}}^{2}\right]<\infty,
\end{align*}
where the second line is due to (\ref{holder1}) with 
\begin{align*}
    & u_{1}=\left|\sum\limits_{i=1}^{n-1}|\int_{t_{i}}^{t_{i+1}}\mathcal{L}J^{\theta}(t,X_{t-})dt+\int_{t_{i}}^{t_{i+1}}\frac{\partial J^{\theta}}{\partial x}(t,X_{t-})^{\dagger}\sigma(t,X_{t-})dW_{t}|\right.\\
    & \left.\cdot|\int_{t_{i-1}}^{t_{i}}\mathcal{L}J^{\theta}(t,X_{t-})dt+\int_{t_{i-1}}^{t_{i}}\frac{\partial J^{\theta}}{\partial x}(t,X_{t-})^{\dagger}\sigma(t,X_{t-})dW_{t}|\right|,\\
    & u_{2}=\left[\frac{2}{\pi}\int_{0}^{T}\left(\frac{\partial J^{\theta}}{\partial x}(t,X_{t-})^{\dagger}\sigma(t,X_{t-})\right)^{2}dt\right],v_{1}=v_{2}=1,p=q=2,
\end{align*}
the fifth line is due to (\ref{holder1}) with $u_{i}=\int_{t_{i}}^{t_{i+1}}\mathcal{L}J^{\theta}(t,X_{t-})dt+\int_{t_{i}}^{t_{i+1}}\frac{\partial J^{\theta}}{\partial x}(t,X_{t-})^{\dagger}\sigma(t,X_{t-})dW_{t},v_{i}=\int_{t_{i-1}}^{t_{i}}\mathcal{L}J^{\theta}(t,X_{t-})dt+\int_{t_{i-1}}^{t_{i}}\frac{\partial J^{\theta}}{\partial x}(t,X_{t-})^{\dagger}\sigma(t,X_{t-})dW_{t},p=q=2$ and (\ref{holder3}) with $f=\left(\frac{\partial J^{\theta}}{\partial x}(t,X_{t-})^{\dagger}\sigma(t,X_{t-})\right)^{2},$ $g=1,p=q=2$, the eighth line is due to adding extra positive terms, the tenth line is due to (\ref{holder1}) with $u_{i}=\left|\int_{t_{i}}^{t_{i+1}}\mathcal{L}J^{\theta}(t,X_{t-})dt+\int_{t_{i}}^{t_{i+1}}\frac{\partial J^{\theta}}{\partial x}(t,X_{t-})^{\dagger}\sigma(t,X_{t-})dW_{t}\right|^{2},v_{i}=1,p=q=2$, the twelfth line is due to $(u_{1}+u_{2})^{4}\leq 8(u_{1}^{4}+u_{2}^{4})$ with $u_{1}=\int_{t_{i}}^{t_{i+1}}\mathcal{L}J^{\theta}(t,X_{t-})dt$ and $u_{2}=\int_{t_{i}}^{t_{i+1}}\frac{\partial J^{\theta}}{\partial x}(t,X_{t-})^{\dagger}\sigma(t,X_{t-})dW_{t}$, the fourteenth line is due to (\ref{holder3}) with $f=\mathcal{L}J^{\theta}(t,X_{t-}),g=1,p=4,q=4/3$ and (\ref{bdg}) with $p=4,(M_{i})_{t}=\int_{t_{i}}^{t_{i}+t}\frac{\partial J^{\theta}}{\partial x}(t,X_{t-})^{\dagger}\sigma(t,X_{t-})dW_{t}$ and $C$ is a positive constant, the seventeenth line is due to (\ref{holder3}) with $f=\left|\frac{\partial J^{\theta}}{\partial x}(t,X_{t-})^{\dagger}\sigma(t,X_{t-})\right|^{2},g=1,p=q=2$, and the twenty-first line is due to the fact that $T=n(\Delta t)$.

 By Lemma \ref{lemma7}, $\sup_{n}E|Y_{n}|^{2}<\infty$ implies that $\{|Y_{n}|\}$ is uniformly integrable. Together with $Y_{n}\stackrel{d}\rightarrow 0$, this results in $E|Y_{n}|\rightarrow 0$, i.e., $|\tilde{R}_{1}(\theta)|\rightarrow 0$.

 Now let's prove $|\tilde{R}_{2}(\theta)|\rightarrow 0$. Let
\begin{align*}
    &a_{i}=\int_{t_{i}}^{t_{i+1}}\mathcal{L}J^{\theta}(t,X_{t-})dt+\int_{t_{i}}^{t_{i+1}}\frac{\partial J^{\theta}}{\partial x}(t,X_{t-})^{\dagger}\sigma(t,X_{t-})dW_{t},\\
    &b_{i}=\int_{t_{i}}^{t_{i+1}}\left[J^{\theta}(t,X_{t})-J^{\theta}(t,X_{t-})\right]dN_{t},\\
    &c_{i}=\int_{t_{i-1}}^{t_{i}}\mathcal{L}J^{\theta}(t,X_{t-})dt+\int_{t_{i-1}}^{t_{i}}\frac{\partial J^{\theta}}{\partial x}(t,X_{t-})^{\dagger}\sigma(t,X_{t-})dW_{t},\\
    &d_{i}=\int_{t_{i-1}}^{t_{i}}\left[J^{\theta}(t,X_{t})-J^{\theta}(t,X_{t-})\right]dN_{t}.
\end{align*}
Recall that 
\[\tilde{R}_{2}(\theta)=E\left[\sum_{i=1}^{n-1}|a_{i}+b_{i}||c_{i}+d_{i}|\right]-E\left[\sum_{i=1}^{n-1}|a_{i}||c_{i}|\right].\]
By (\ref{jumpproof2}) and the condition $b_{i}d_{i}=0$, so we can use Lemma \ref{lemma6} with $a_{i},b_{i},c_{i}$ and $d_{i}$, we have the following:
\begin{align*}
    |\tilde{R}_{2}(\theta)|\leq & \sum\limits_{i=1}^{n-1}E\left\{\left[\int_{t_{i}}^{t_{i+1}}\left|J^{\theta}(t,X_{t})-J^{\theta}(t,X_{t-})\right|dN_{t}\right]\right.\\
    &\cdot\left[\left|\int_{t_{i+1}}^{t_{i+2}}\mathcal{L}J^{\theta}(t,X_{t-})dt+\int_{t_{i+1}}^{t_{i+2}}\frac{\partial J^{\theta}}{\partial x}(t,X_{t-})^{\dagger}\sigma(t,X_{t-})dW_{t}\right|\right.\\
    & +\left.\left.\left|\int_{t_{i-1}}^{t_{i}}\mathcal{L}J^{\theta}(t,X_{t-})dt+\int_{t_{i-1}}^{t_{i}}\frac{\partial J^{\theta}}{\partial x}(t,X_{t-})^{\dagger}\sigma(t,X_{t-})dW_{t}\right|\right]\right\},
\end{align*}
where $t_{-1}=t_{0}=0$, and $t_{n+1}=t_{n}=T$ for simplicity.

 Assumption \ref{ass3} implies that for any compact set $\Gamma$, $\sup_{0\leq t\leq T} \left\{E\left|J^{\theta}(t,X_{t})\right|^{2}\right\}^{\frac{1}{2}}$ is bounded for all $\theta\in\Gamma$. Assumption \ref{ass1}(v) states that $\sup\limits_{0\leq i\leq n-1}P\left(N_{t_{i+1}}-N_{t_{i}}=1\right)=\mathcal{O}(\Delta t)$. Hence, for all $\theta\in\Gamma$, by Lemma \ref{lemma9}, we have the following:
\begin{align*}
    |\tilde{R}_{2}(\theta)|&\leq 8\left\{\sup_{0\leq i\leq n-1}P\left(N_{t_{i+1}}-N_{t_{i}}=1\right)\cdot\sup_{0\leq t\leq T} E\left[J^{\theta}(t,X_{t})\right]^{2}\cdot T\right\}^{1/2}\Vert\mathcal{L}J^{\theta}(\cdot,X_{\cdot})\Vert_{L^{2}}\\
    &+8(C\Delta t)^{\frac{1}{4}}\left\{\sup_{0\leq t\leq T} E\left|J^{\theta}(t,X_{t})\right|^{2}\right\}^{\frac{1}{2}}\left[\left\Vert\left|\frac{\partial J^{\theta}}{\partial x}(\cdot,X_{\cdot})^{\dagger}\sigma(\cdot,X_{\cdot})\right|^{2}\right\Vert_{L^{2}}\right]^{\frac{1}{2}}\\
    &\quad\cdot
    \left[\frac{T}{
    \Delta t}\sup_{0\leq i\leq n-1}P\left(N_{t_{i+1}}-N_{t_{i}}=1\right)\right]^{\frac{3}{4}}\rightarrow 0\quad\text{ as }\Delta t\rightarrow 0.
\end{align*}
The desired result follows from Lemma \ref{lemma3}.
\end{proof}

\end{document}